\documentclass[pdftex,pdfencoding=auto]{article}

\usepackage[accepted]{icml2021}

\usepackage{microtype}
\usepackage{graphicx}
\usepackage{subfigure}
\usepackage{times}
\usepackage{pdfcomment}
\usepackage{booktabs} 
\usepackage[T1]{fontenc} 
\usepackage{physics} 
\usepackage[dvipsnames]{xcolor}
\usepackage{amsmath,amsthm,verbatim,amssymb,amsfonts,amscd,mathrsfs} 
\usepackage[capitalise]{cleveref}
\usepackage[algo2e,ruled]{algorithm2e}
\usepackage{multirow}
\usepackage{mathtools}
\usepackage{makecell}  
\usepackage{hhline}  
\usepackage{url}
\usepackage{todonotes} 
\usepackage{booktabs} 
\usepackage[inline]{enumitem} 

\newtheorem{obs}{Observation}
\newtheorem{theorem}{Theorem}
\newtheorem{rem}{Remark}
\newtheorem*{rem*}{Remark}
\newcommand{\defeq}{\vcentcolon=}

\crefname{equation}{Eq.}{Eqs.}
\crefname{figure}{Fig.}{Figs.}
\crefname{section}{Sec.}{Secs.}
\crefname{subsection}{Sec.}{Secs.}
\crefname{theorem}{Thm.}{Thms.}
\crefname{appendix}{Appx.}{Appx.}
\crefname{lemma}{Lemma}{Lemmas}
\crefname{algocf}{Alg.}{Algs.}
\Crefname{algocf}{Algorithm}{Algorithms}

\newif\ifcomments
\commentstrue
\ifcomments\newcommand{\comments}[1]{#1}\else\newcommand{\comments}[1]{}\fi

\definecolor{clrgp}{rgb}{.9,0,.9}
\definecolor{red}{rgb}{.8,0,0}
\definecolor{blue}{rgb}{0,.8, 0}
\definecolor{gray}{rgb}{0.41, 0.41, 0.41}
\definecolor{forestgreen}{rgb}{0.13, 0.55, 0.13}
\definecolor{subtle}{RGB}{152,78,163}

\icmltitlerunning{Bias-Free Scalable Gaussian Processes via Randomized Truncations}

\DeclareMathOperator*{\argmin}{arg\,min}
\DeclareMathOperator*{\expectedvalue}{\mathbb{E}}

\newcommand{\bigo}[1]{\ensuremath{\mathchoice{\mathcal O \! \left( #1 \right)}}{\mathcal O(#1)}{\mathcal O(#1)}{\mathcal O(#1)}}

\newcommand{\reals}{\ensuremath{\mathbb{R}}}

\newcommand{\normaldist}[2]{\ensuremath{\mathchoice{\mathcal{N} \left( #1, #2 \right) }{ \mathcal{N}(#1,#2) }{ \mathcal{N}(#1,#2) }{}}}

\newcommand{\Evover}[2]{\ensuremath{\expectedvalue_{#1} \left[ #2 \right]}}

\newcommand{\intd}[1]{\,\mathrm{d}{#1}}

\newif\ifboldmatrix
\boldmatrixtrue

\ifboldmatrix\newcommand{\boldmatrix}[1]{\mathbf{#1}}\else\newcommand{\boldmatrix}[1]{#1}\fi
\ifboldmatrix\newcommand{\boldgreekmatrix}[1]{\boldsymbol{#1}}\else\newcommand{\boldgreekmatrix}[1]{#1}\fi
\ifboldmatrix\else\fi

\newcommand{\bb}{\ensuremath{\mathbf{b}}}

\newcommand{\bd}{\ensuremath{\mathbf{d}}}
\newcommand{\be}{\ensuremath{\mathbf{e}}}

\newcommand{\bomega}{\ensuremath{\boldsymbol{\omega}}}

\newcommand{\btheta}{\ensuremath{\boldsymbol{\theta}}}

\newcommand{\bw}{\ensuremath{\mathbf{w}}}
\newcommand{\bx}{\ensuremath{\mathbf{x}}}
\newcommand{\by}{\ensuremath{\mathbf{y}}}
\newcommand{\bz}{\ensuremath{\mathbf{z}}}

\newcommand{\bA}{\ensuremath{\boldmatrix{A}}}

\newcommand{\bD}{\ensuremath{\boldmatrix{D}}}

\newcommand{\bI}{\ensuremath{\boldmatrix{I}}}
\newcommand{\bK}{\ensuremath{\boldmatrix{K}}}

\newcommand{\bLambda}{\ensuremath{\boldgreekmatrix{\Lambda}}}

\newcommand{\bP}{\ensuremath{\boldmatrix{P}}}

\newcommand{\bphi}{\ensuremath{\boldsymbol{\phi}}}

\newcommand{\bpsi}{\ensuremath{\boldsymbol{\psi}}}
\newcommand{\bQ}{\ensuremath{\boldmatrix{Q}}}

\newcommand{\bT}{\ensuremath{\boldmatrix{T}}}

\newcommand{\bV}{\ensuremath{\boldmatrix{V}}}
\newcommand{\bW}{\ensuremath{\boldmatrix{W}}}
\newcommand{\bX}{\ensuremath{\boldmatrix{X}}}

\newcommand{\trainK}{\ensuremath{\widehat{\bK}_{\bX\bX}}}

\newcommand{\invquad}{\ensuremath{\by^{\top} \trainK^{-1} \by}}
\newcommand{\logdet}{\ensuremath{\log \vert \trainK \vert}}

\begin{document}

\twocolumn[
\icmltitle{Bias-Free Scalable Gaussian Processes via Randomized Truncations}



\icmlsetsymbol{equal}{*}

\begin{icmlauthorlist}
\icmlauthor{Andres Potapczynski}{equal,zucker}
\icmlauthor{Luhuan Wu}{equal,columbia}
\icmlauthor{Dan Biderman}{equal,zucker}
\icmlauthor{Geoff Pleiss}{zucker}
\icmlauthor{John P. Cunningham}{zucker,columbia}

\end{icmlauthorlist}

\icmlaffiliation{columbia}{Statistics Department, Columbia University}
\icmlaffiliation{zucker}{Zuckerman Institute, Columbia University}

\icmlcorrespondingauthor{Andres Potapczynski}{ap3635@columbia.edu}
\icmlcorrespondingauthor{Luhuan Wu}{lw2827@columbia.edu}
\icmlcorrespondingauthor{Dan Biderman}{db3236@columbia.edu}

\icmlkeywords{Machine Learning, ICML}

\vskip 0.3in
]


\printAffiliationsAndNotice{\icmlEqualContribution} 

\begin{abstract}
Scalable Gaussian Process methods are computationally attractive, yet introduce modeling biases that require rigorous study.
This paper analyzes two common techniques: early truncated conjugate gradients (CG) and random Fourier features (RFF).
We find that both methods introduce a systematic bias on the learned hyperparameters: CG tends to underfit while RFF tends to overfit.
We address these issues using randomized truncation estimators that eliminate bias in exchange for increased variance.
In the case of RFF, we show that the bias-to-variance conversion is indeed a
trade-off: the additional variance proves detrimental
to optimization. However, in the case of CG, our unbiased learning procedure
meaningfully outperforms its biased counterpart with minimal additional computation.
Our code is available at \url{https://github.com/cunningham-lab/RTGPS}.
\end{abstract}

\section{Introduction}
\label{intro}


Gaussian Processes (GP) are popular and expressive nonparametric models, and considerable effort has gone into alleviating their cubic runtime complexity.
Notable successes include
inducing point methods \citep[e.g.][]{snelson2006sparse,titsias2009variational,hensman2013gaussian},
finite-basis expansions \citep[e.g.][]{rahimi2008random,mutny2019efficient,wilson2020efficiently, loper2020general},
nearest neighbor truncations \citep[e.g.][]{datta2016hierarchical,katzfuss2021general},
and iterative numerical methods \citep[e.g.][]{cunningham2008fast,cutajar2016preconditioning,gardner2018gpytorch}.
Common to these techniques is the classic \emph{speed-bias tradeoff}: coarser GP approximations afford faster but more biased solutions that in turn affect both the model's predictions and learned hyperparameters.
While a few papers analyze the bias of inducing point methods \citep{bauer2017understanding,burt2019rates}, the biases of other approximation techniques, and their subsequent impact on learned GP models, have not been rigorously studied.

Here we scrutinize the biases of two popular techniques -- random Fourier features (RFF) \citep{rahimi2008random} and conjugate gradients (CG) \citep[e.g.][]{cunningham2008fast,cutajar2016preconditioning,gardner2018gpytorch}. These methods are notable due to their popularity and because they allow dynamic control of the speed-bias tradeoff: at any model evaluation, the user can adjust the number of CG iterations or RFF features to a desired level of approximation accuracy. In practice, it is common to truncate these methods to a fixed number of iterations/features that is deemed adequate. However, such truncation will stop short of an exact (machine precision) solution and potentially lead to biased optimization outcomes.  

We provide a novel theoretical analysis of the biases resulting from RFF and CG on the GP log marginal likelihood objective.
Specifically, we prove that CG is biased towards hyperparameters that underfit the data, while RFF is biased towards overfitting.  In addition to yielding suboptimal hyperparameters, these biases hurt posterior predictions, regardless of the inference method used at test-time.  Perhaps surprisingly, this effect is not subtle, as we will demonstrate.
Our analysis suggests there is value in debiasing GP learning with CG and RFF.  To do so, we turn to recent work that shows the merits of exchanging the speed-bias tradeoff for a \emph{speed-variance tradeoff} \citep{beatson2019efficient,chen2020residual,luo2020sumo, oktay2020randomized}. 
These works all introduce a randomization procedure that reweights elements of a fast truncated estimator, eliminating its bias at the cost of increasing its variance.

We thus develop bias-free versions of GP learning with CG and RFF using randomized truncation estimators.  In short, we randomly truncate the number of CG iterations and RFF features, while reweighting intermediate solutions to maintain unbiasedness.  Our variant of CG uses the {\bf R}ussian {\bf R}oulette estimator \cite{kahn1955use}, while our variant of RFF uses the {\bf S}ingle {\bf S}ample estimator of \citet{lyne2015russian}.
We believe our {\bf RR-CG} and {\bf SS-RFF} methods to be the ﬁrst to produce unbiased estimators of the GP log marginal likelihood with  $< \bigo{N^3}$ computation.



Finally, through extensive empirical evaluation, we find our methods and their biased counterparts indeed constitute a bias-variance tradeoff. Both RR-CG and SS-RFF are unbiased, recovering nearly the same optimum as the exact GP method, while GP trained with CG and RFF often converge to solutions with worse likelihood. We note that bias elimination is not always practical. For SS-RFF, the optimization is slow, due to the large auxiliary variance needed to counteract the slowly decaying bias of RFF. On the other hand, RR-CG incurs a minimal variance penalty, likely due to the favorable convergence properties of CG. In a wide range of benchmark datasets, RR-CG demonstrates similar or better predictive performance compared to CG using the same expected computational time. 

To summarize, this work offers three main contributions: 
\begin{itemize}[noitemsep,topsep=0pt,parsep=0pt,partopsep=0pt]
    \item theoretical analysis of the bias of CG- and RFF-based GP approximation methods (\S 3)
    \item RR-CG and SS-RFF: bias-free versions of these popular GP approximation methods (\S 4)
    \item results demonstrating the value of our RR-CG and SS-RFF methods (\S 3 and 5).
\end{itemize}


\section{Background}
\label{background}

We consider observed data $\mathcal{D} = \{(\bx_i, y_i)\}_{i=1}^N$ for $\bx_i \in \mathbb{R}^d$ and $y_i \in \mathbb{R}$, and the standard GP model:
\begin{equation*}
\begin{gathered}
    f(\cdot) \sim \mathcal{GP} ( \mu(\cdot), k(\cdot, \cdot)), \\
\by_i = f(\bx_i) + \epsilon_i, \qquad \epsilon_i \sim \normaldist{0}{\sigma^2}
\end{gathered}
\end{equation*}

where $k(\cdot, \cdot)$ is the covariance kernel, $\mu(\cdot)$ is set to zero without loss of generality, and hyperparameters are collected into the vector $\btheta$, which is optimized as:
  \begin{equation}
  \begin{aligned}
   \btheta^* &= {\textstyle \argmin_{\btheta} } \: \mathcal{L}(\btheta) 
   \\
    \mathcal{L}(\btheta)  &=  - \log p( \by \! \mid \! \bX; \btheta) \\
    & = \frac{1}{2} \big(   \underbrace{\log \vert
    \trainK \vert}_{\textrm{model complexity}} + \underbrace{\by^{\! \top}
    \trainK^{-1} \by}_{\textrm{data fit}}  +  N \log 2 \pi \big) 
  \end{aligned}
  \label{eqn:log_lik}
  \end{equation}

where $\trainK \in \reals^{N \times N}$ is the Gram matrix of all data points with diagonal observational noise:
$$
\trainK[i, j] = k(\bx_i, \bx_j) + \sigma^2 \mathbb{I}_{i = j}.
$$
Following standard practice, we optimize $\btheta$ with gradients: 
  %
  \begin{align}
    {\textstyle \frac{\partial \mathcal{L}}{\partial \btheta} } &= \frac{1}{2} \left(
     \tr{\trainK^{-1} {\textstyle \frac{\partial \trainK}{\partial \btheta}} } -
     \by^{\top} \trainK^{-1} {\textstyle \frac{\partial \trainK}{\partial \btheta}} \trainK^{-1} \by \right).
    \label{eqn:log_lik_deriv}
  \end{align}
Three terms thus dominate the computational complexity: 
$\by^\top \trainK^{-1} \by$, $\log \vert \trainK
\vert$, and $\text{tr} \{ \trainK^{-1}\frac{\partial \trainK}{\partial \btheta} \}$. 
The common approach to computing this triad is the Cholesky factorization, requiring
$\bigo{N^3}$ time and $\bigo{N^2}$ space.

Extensive literature has accelerated the inference and hyperparameter learning of GP. Two very popular strategies are using \emph{conjugate gradients}  \citep{cunningham2008fast, cutajar2016preconditioning,
gardner2018gpytorch, wang2019exact} to approximate the linear solves in \cref{eqn:log_lik_deriv}, and  
\emph{random Fourier features} \citep{rahimi2008random}, which constructs a randomized finite-basis approximation of the kernel.

\subsection{Conjugate Gradients}
\label{sec:CG_background}
To apply conjugate gradients to GP learning, we begin by replacing the gradient in \cref{eqn:log_lik_deriv} with a stochastic estimate 
\citep{cutajar2016preconditioning,gardner2018gpytorch}:
%
\begin{align}
  {\textstyle \frac{\partial \mathcal{L}}{\partial \btheta} }
  &\approx
   \frac{1}{2} \left( \bz^\top \trainK^{-1} {\textstyle \frac{\partial \trainK}{\partial \btheta}} \bz -
   \by^{\top} \trainK^{-1} {\textstyle \frac{\partial \trainK}{\partial \btheta}} \trainK^{-1} \by \right),
  \label{eqn:log_lik_deriv_stochastic}
\end{align}
where $\bz$ is a random variable such that $\mathbb{E}[\bz] = 0$ and $\mathbb{E}[\bz\bz^T] = \bI$.
Note that the first term constitutes a stochastic estimate of the trace term \cite{hutchinson1989stochastic}.
Thus, stochastic optimization of \cref{eqn:log_lik} can be reduced to computing the linear solves $\trainK^{-1} \by$ and $\trainK^{-1} \bz$.

Conjugate gradients (CG) \cite{hestenes1952methods} is an iterative algorithm for solving positive definite linear systems $\bA^{-1} \bb$.
It consists of a three-term recurrence, where each new term requires only a matrix-vector multiplication with $\bA$.
More formally, each CG iteration computes a new term of the following summation:
\begin{align}
  \textstyle
  \bA^{-1} \bb = \sum_{i=1}^N \gamma_i \bd_i, \qquad
  \label{eqn:cg_series}
\end{align}
where the $\gamma_i$ are coefficients and the $\bd_i$ are conjugate search directions \cite{golub2012matrix}.
$N$ iterations of CG produce all $N$ summation terms and recover the exact solution. In practice, exact convergence may require more than $N$ iterations due to inaccuracies of floating point arithmetic.
However, the summation converges exponentially, and so $J \ll N$ iterations may suffice to achieve high accuracy.

CG is an appealing method for computing $\trainK^{-1} \by$ and $\trainK^{-1} \bz$ due to its computational complexity and its potential for GPU-accelerated matrix products.
$J$ iterations takes at most $\bigo{JN^2}$ time 
and $\bigo{N}$ space if the matrix-vector products are performed in a map-reduce fashion \cite{wang2019exact}.
However, ill-conditioned kernel matrices hinder the convergence rate \cite{cutajar2016preconditioning}, and so the $J^\text{th}$ CG iteration may yield an inaccurate approximation of \cref{eqn:log_lik_deriv_stochastic}.


\subsection{Random Fourier Features}
\label{RFF_background}
\citet{rahimi2008random} introduce a randomized finite-basis approximation to
stationary kernels:
  \begin{equation}
      k(\bx, \bx') = k(\bx - \bx') \approx \bphi(\bx)^{\top}\bphi(\bx')
  \end{equation}
where $\bphi(\bx) \in \reals^J$ and $J \ll N$.
The RFF approximation relies on Bochner's theorem \citep{bochner1959lectures}: letting
$\tau=\bx-\bx'$, all stationary kernels $k(\tau)$ on $\mathbb{R}^d$ can be exactly expressed as the Fourier dual of a nonnegative measure $\mathbb{P}(\bomega)$:
$
      k(\tau) = \textstyle \int \mathbb{P}(\bomega) \exp(i\bomega\tau) \, d\bomega.
      \label{eqn:bochner}
$
A Monte Carlo approximation of this Fourier transform yields:
\begin{align*}
  k(\tau) \approx \frac{2}{J} \sum_{j=1}^{J/2} \exp(i\bomega_j\tau), \quad
  \bomega_j \sim \mathbb{P}(\bomega),
\end{align*}
which simplifies to a finite-basis approximation:
  %
%
\begin{gather*}
  \bK_{\bX\bX} \approx [\bphi(\bx_1) \ldots \bphi(\bx_n)] [\bphi(\bx_1) \ldots \bphi(\bx_n)]^\top,
  \\
  \bphi(\bx) = [ \cos(\bomega_i^\top \bx), \: \sin(\bomega_i^\top \bx)]_{i=1}^{J/2},
  \quad
  \bomega_i \sim \mathbb{P}(\bomega).
\end{gather*}
For many common kernels, $\mathbb{P}(\bomega)$ can be computed
in closed-form (e.g. RBF kernels have zero-mean Gaussian spectral densities).
The approximated log likelihood can be computed in
$\bigo{J^3 + N}$ time and $\bigo{JN}$ space using the Woodbury inversion lemma and the matrix determinant
lemma, respectively.
The number of random features $J/2$ is a user choice, with typical values between
100-1000.
More features lead to more accurate kernel approximations.

\begin{figure*}[ht]
\vskip 0.1in
\begin{center}
\includegraphics[scale=0.4]{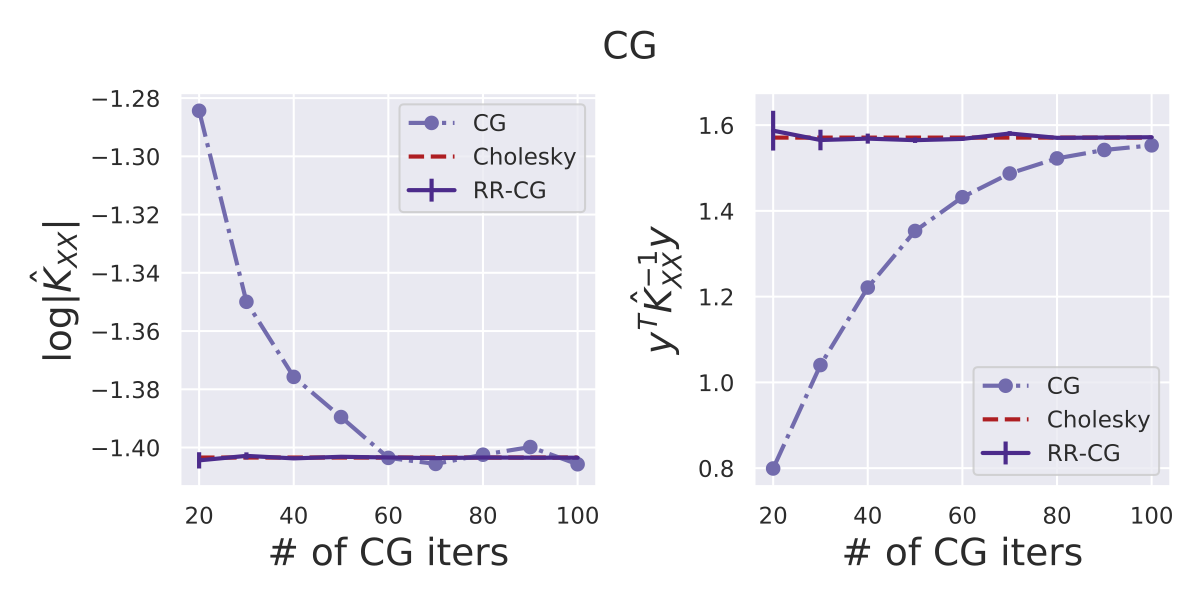}
\includegraphics[scale=0.4]{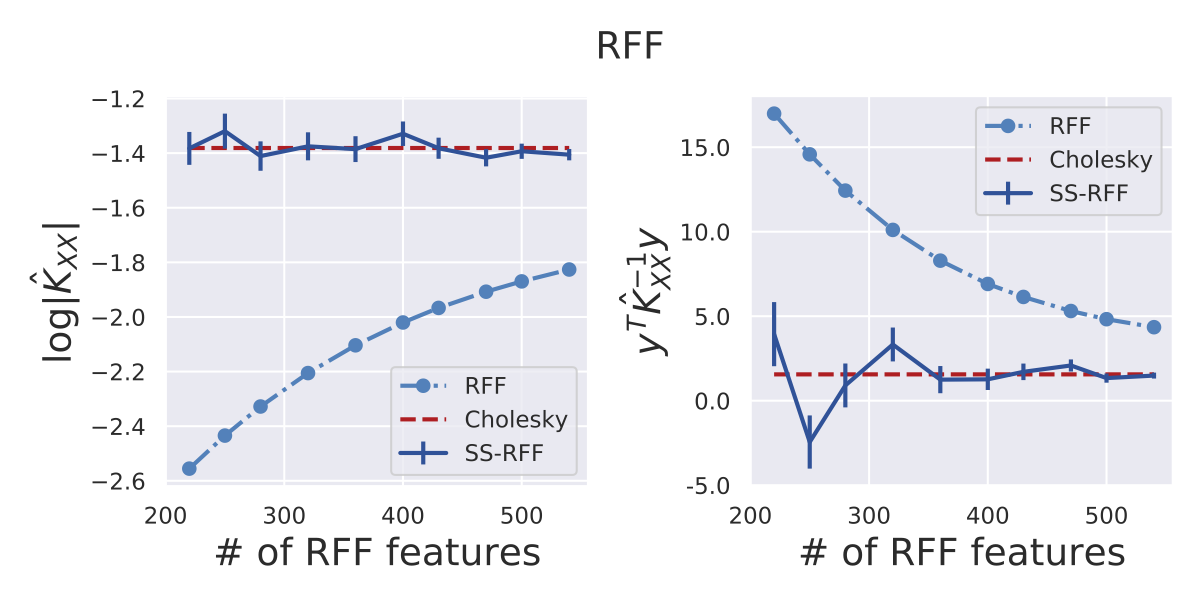}
 \caption{
  CG (left) systematically overestimates $\logdet$ and underestimates $\invquad$ whereas RFF (right)
  does the opposite. The dashed orange line shows the exact values computed by Cholesky. Our unbiased methods, RR-CG (left) and SS-RFF (right), recover the true $\logdet$ and $\invquad$ values.
  For these two methods, the $x$-axis indicates the expected number of iterations/features.}
\label{fig:debias_cg_rff}
\end{center}
\vskip -0.2in
\end{figure*}

\subsection{Unbiased Randomized Truncation}\label{subsec:randomized_truncation_intro}
    We will now briefly introduce Randomized Truncation Estimators, which are the primary tool we use to unbias the CG and RFF log marginal likelihood estimates.
    At a high level, assume that we wish to estimate some quantity $\psi$ that can be expressed as a (potentially-infinite) series: $$
      \textstyle \psi = \sum_{j=1}^H \Delta_j, \quad H \in \mathbb{N} \cup \{\infty\}.
    $$
    Here and in the following sections, $\Delta_j$ can either be random or deterministic.
    To avoid the expensive evaluation of the full summation, a randomized truncation estimator chooses a random term $J \in \{1, \ldots, H\}$ with probability mass function $\mathbb{P}(J) = \mathbb{P}(\mathcal{J}=J)$ after which to truncate computation.
    In the following, we introduce two means of deriving unbiased estimators by upweighting the summation terms. 

    \textbf{The Russian Roulette estimator} \citep{kahn1955use} 
    obtains an unbiased estimator $\bar{\psi}_J$ by truncating the sum after $J \sim \mathbb{P}(J)$ terms and dividing the surviving terms by their survival probabilities:
    \begin{equation} \label{eqn:rr}
        \bar {\psi}_J
        = \sum_{j=1}^{J} \frac{\Delta_j}{\mathbb{P}(\mathcal{J} \geq j)}
        = \sum_{j=1}^{H} \left( \frac{\mathbb{I}_{J \geq j}}{\mathbb{P}(\mathcal{J} \geq j)} \right) \Delta_j, 
    \end{equation}
    and, $
        \mathbb{E}[ \bar{\bpsi}_J ] = \sum_{i=1}^{H} \Delta_j =  \psi.
    $
    (See appendix for further derivation.)
    The choice of $\mathbb P (J)$ determines both the computational efficiency and the variance of $\bar {\psi}_J$.
    A thin-tailed $\mathbb P(J)$ will often truncate sums after only a few terms ($J \ll H$).
    However, tail events $(J \approx H)$ are upweighted inversely to their low survival probability, and so thin-tailed truncation distributions may lead to high variance.
    
    \textbf{The Single Sample estimator} \citep{lyne2015russian} implements an alternative reweighting scheme. 
    After drawing $J \sim \mathbb{P}(J)$, it computes a single summation term $\Delta_J$, which it upweights by $1/\mathbb{P}(J)$:
    \begin{equation}
        \bar{\psi}_J = \frac{\Delta_J}{\mathbb{P}(J)}
                = \sum_{j=1}^{H} \left( \frac{\mathbb{I}_{J = j}}{\mathbb{P}(\mathcal{J} = j)} \right) \Delta_j .
        \label{eqn:ss_estimator}
    \end{equation}
    This procedure is unbiased, and it amounts to estimating $\psi$ using a single importance-weighted sample from $\mathbb{P}(J)$ (see appendix).
    Again, $\mathbb{P}(J)$ controls the speed/variance trade-off.
    We refer the reader to \cite{beatson2019efficient} for a detailed comparison of these two estimators. We emphasize that both estimators remain unbiased even if $\Delta_j$ is a random variable, as long as it is independent from the random truncation integer $J$.

\begin{figure}[ht]
\vskip 0.1in
\begin{center}
\centerline{\includegraphics[scale=0.4]{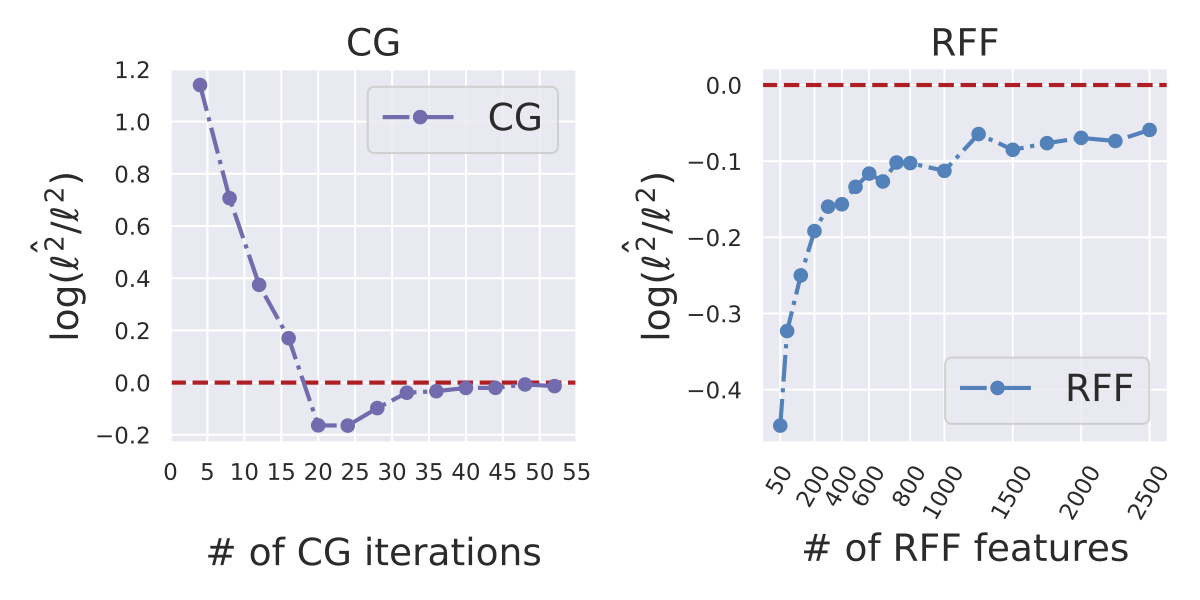}}
\caption{
  Kernel lengthscale values learned by optimizing (biased) CG and RFF log marginal 
  likelihood approximations.
  CG overestimates the optimal kernel lengthscale whereas RFF
  underestimates it.  We plot the divergence (in log-ratio scale) between the learned
  and true lengthscales as a function of the number of CG iterations (left) and of
  the number of RFF samples (right).}
\label{lengthscale-bias}
\end{center}
\vskip -0.2in
\end{figure}

 \section{GP Learning with CG and RFF is Biased}\label{sec:bias}
  Here we prove that early truncated CG and RFF provide biased approximations to the terms comprising the GP log marginal  likelihood (Eq.~\ref{eqn:log_lik}). We also derive the bias decay rates for each method. 
  We then empirically demonstrate these biases and show they affect the hyperparameters learned through optimization.
  Remarkably, we find that the above biases are \emph{highly systematic}: CG-based GP learning favors underfitting hyperparameters while RFF-based learning favors overfitting hyperparameters. 
 
  \subsection{CG Biases GP Towards Underfitting}
  \label{subsec:CG_bias}
    In the GP literature, CG has often been considered an ``exact'' method for computing the log marginal likelihood \cite{cunningham2008fast,cutajar2016preconditioning}, as the iterations are only truncated after reaching a pre-specified residual error threshold (e.g. $10^{-10}$).
    However, as CG is applied to ever-larger kernel matrices it is common to truncate the CG iterations before reaching this convergence \cite{wang2019exact}.
    While this accelerates the hyperparameter learning process, the resulting solves and gradients can no longer be considered ``exact.''
    In what follows, we show that the early-truncated CG optimization objective is not only approximate but also systematically biased towards underfitting.
    
    To analyze the early-truncation bias, we adopt the analysis of \citet{gardner2018gpytorch} that recovers the GP log marginal likelihood (Eq.~\ref{eqn:log_lik}) from the stochastic gradient's CG estimates of $\trainK^{-1} \by$ and $\trainK^{-1} \bz$ (Eq.~\ref{eqn:log_lik_deriv_stochastic}).
    Recall the two terms in the log marginal likelihood are the ``data fit'' term $\by^\top \trainK^{-1} \by$ and the ``model complexity'' term $\logdet$.
    The first term falls directly out of the CG estimate of $\trainK^{-1} \by$, while a stochastic estimate of $ \logdet $ can be obtained through the byproducts of CG's computation for  $\trainK^{-1} \bz$.
    \citet{gardner2018gpytorch} show that the CG coefficients in \cref{eqn:cg_series}
    can be manipulated to produce a partial tridiagonalization: $\bT^{(J)}_{\bz} = \bQ^{(J)\top}_\bz \trainK \bQ^{(J)}_\bz,$ where $\bT^{(J)}_\bz \in \reals^{J \times J}$ is tridiagonal.
    $\bT^{(J)}_\bz$ can compute the Stochastic Lanczos Quadrature estimate of $\trainK$ \citep{ubaru2017fast, dong2017scalable}:
    \begin{align}\label{eqn:logdet_estimate}
      \log |\trainK | &= \mathbb{E} \left[ \bz^T (\log \trainK) \bz \right]
      \nonumber
      \\
      &\approx
      \Vert \bz \Vert^2 \be_1^\top  \left( \log \bT_\bz^{(J)} \right) \be_1,
    \end{align}
    where $\log (\cdot )$ is the matrix logarithm and $\be_1$ is the first row of the identity matrix.
    The following theorem analyzes the bias of these $\by^\top \trainK^{-1} \by$ and $\logdet$ estimates:
 \begin{theorem}\label{CG_bias_theorem}
  Let $u_J$ and $v_J$ be the estimates of $\by^\top \trainK^{-1} \by$ and $\logdet$ respectively after $J $ iterations of CG; i.e.:
  \begin{align*}
     u_J = \by^\top \left({\textstyle \sum_{i=1}^J } \gamma_i \bd_i \right), \quad
     v_J = \Vert \bz \Vert^2 \be_1^\top \left(\log \bT_\bz^{(J)} \right) \be_1.
  \end{align*}
  If $J<N$, CG underestimates the inverse quadratic term and overestimates the log determinant in expectation:
  \begin{equation}
        u_J
        \leq
        \invquad,
        \quad
        \mathbb{E}_{\bz} [ v_J ]  \geq \logdet.
  \end{equation}
  The biases of both terms decay at a rate of $\bigo{C^{-2J}}$, where $C$ is a constant that depends on the conditioning of $\trainK$.
\end{theorem}
\textit{Proof sketch.}
The direction of the biases can be proved using a connection between CG and numeric quadrature.
$u_J$ and $v_J$ are exactly equal to the $J$-point Gauss quadrature approximation of $\by^\top \trainK^{-1} \by$ and $\Vert \bz \Vert^2 \be_1^\top ( \log \bT^{(N)}_\bz ) \be_1$ represented as Riemann-Stieltjes integrals.
The sign of the CG approximation bias follows from the standard Gauss quadrature error bound, which is negative for $\by^\top \trainK^{-1} \by$ and positive for $\logdet$.
The convergence rates are from standard bounds on CG \cite{golub2012matrix} and the analysis of \citet{ubaru2017fast}.
See appendix for a full proof.

\cref{fig:debias_cg_rff} confirms our theoretical analysis and demonstrates the systematic
biases of CG. We plot the log marginal likelihood terms for a subset of the PoleTele UCI dataset, varying the number of CG iterations ($J$) used to produce the estimates.
Compared against the exact terms (computed with Cholesky), we see an overestimation of $\logdet$ and an
underestimation of $\invquad$.
These biases are most prominent when using few CG iterations.

We turn to study the effect of CG learning on hyperparameters.
Since the log marginal likelihood is a nonconvex function of $\btheta$, it is not possible to directly prove how the bias affects $\btheta$.
Nevertheless, we know intuitively that underestimating $\invquad$ de-prioritizes model fit while overestimating $\logdet$ over-penalizes model complexity.
Thus, the learned hyperparameters will likely underfit the data.
Underfitting may manifest in an overestimation of the learned lengthscale $\ell$, as low values of $\ell$ increase the flexibility and the complexity of the model.
This hypothesis is empirically confirmed in \cref{lengthscale-bias} (left panel).
We train a GP regression
model on a toy dataset: $y = x \sin(5 \pi x) + \varepsilon$ and $\varepsilon
\sim \mathcal{N}(0, 0.01)$.
We fix all hyperparameters other than the lengthscale, which is learned using both CG-based optimization and (exact) Cholesky-based optimization. The overestimation of $\ell$ decays with the number of CG iterations.

%
\subsection{RFF Biases GP Towards Overfitting}
\label{subsec:RFF_bias}
Previous work has studied the accuracy of RFF's approximation to the entries of the Gram matrix $k(\bx, \bx') \approx \bphi(\bx)^\top \bphi(\bx')$ \citep{rahimi2008random,sutherland2015error}.
However, to the best of our knowledge there has been little analysis of nonlinear functions of this approximate Gram matrix, such as $\invquad$ and $\logdet$ appearing in the GP objective.
Interestingly, we find that RFF systematically biases these terms: 
\begin{theorem}\label{RFF_bias_theorem}
  Let $\widetilde{\bK}_{J}$ be the RFF approximation with $J/2$ random features. In expectation, $\widetilde{\bK}_{J}$ overestimates the inverse quadratic and underestimates the log determinant:
  \begin{equation}
        {\textstyle \Evover{\mathbb{P}(\bomega)}{\by^{\top} \widetilde{\bK}_{J}^{-1} \by} }
        \geq
        \invquad
  \end{equation}
  \begin{equation}
        {\textstyle \Evover{\mathbb{P}(\bomega)}{\log |\widetilde{\bK}_{J}|} }
        \leq
        \logdet.
  \end{equation}
  The biases of both terms decay at a rate of $\bigo{1/J}$.
\end{theorem}
\textit{Proof sketch.}
The direction of the biases is a straightforward application of Jensen's inequality, noting that $\trainK^{-1}$ is a convex function and $\logdet$ is a concave function.
The magnitude of the bias is derived from a second-order Taylor expansion that closely resembles the analysis of \citet{nowozin2018debiasing}.
See appendix for full proof.

Again, \cref{fig:debias_cg_rff} confirms the
systematic biases of RFF, which decay at a rate proportional to the number of features, as predicted by \cref{RFF_bias_theorem}.
Hence, RFF should affect the learned hyperparameters in a manner opposite to CG. Overestimating $\invquad$ emphasizes data fitting while underestimating $\logdet$ reduces the model complexity penalty, overall resulting in overfitting behavior. Following the intuition presented in \cref{subsec:CG_bias}, we expect the lengthscale to be underestimated, as empirically confirmed by \cref{lengthscale-bias} (right panel). The figure also illustrates the slow decay of the RFF bias.

\section{Bias-free Scalable Gaussian Processes}
    We debias the estimates of both the GP training objective in \cref{eqn:log_lik} and its gradient in \cref{eqn:log_lik_deriv} (as approximated by CG and RFF) using unbiased randomized truncation estimators.
    To see how such estimators apply to GP hyperparameter learning, we note that both CG and RFF recover the true log marginal likelihood (or an unbiased estimate thereof) in their limits:
    \begin{obs}
        CG recovers the exact log marginal likelihood in expectation in at most $N$ iterations:
        \begin{align}
            \by^\top \trainK^{-1} \by &=  \by^\top  \left( \textstyle {\sum_{j=1}^N } \gamma_j \bd_j \right),
            \label{eqn:lim_cg_invquad}
            \\
            \logdet &=  {\textstyle \Evover{\bz}{\Vert \bz \Vert^2 \be_1^\top (\log \bT^{(N)}_\bz) \be_1} }.
            \label{eqn:lim_cg_logdet}
        \end{align}
        By the law of large numbers, RFF converges almost surely to the exact log marginal likelihood as the number of random features goes to infinity:
        \begin{align}
            \by^\top \trainK^{-1} \by &= \lim_{J \to \infty} \by^\top \widetilde{\bK}_J^{-1} \by,
            \label{eqn:lim_rff_invquad}
            \\
            \logdet &= \lim_{J \to \infty} \log \vert \widetilde{\bK}_J \vert.
            \label{eqn:lim_rff_logdet}
        \end{align}
    \end{obs}
    %
    

    To maintain the scalability of CG and RFFs while eliminating bias, we 
    express the log marginal likelihood terms in \cref{eqn:lim_cg_invquad,eqn:lim_cg_logdet,eqn:lim_rff_invquad,eqn:lim_rff_logdet} as summations amenable to randomized truncation.
    We then apply the Russian Roulette and Single Sample estimators of \cref{subsec:randomized_truncation_intro} to avoid computing all summation terms while obtaining the same result in expectation.

\subsection{Russian Roulette-Truncated CG (RR-CG)}
\label{sec:debias_cg}
The stochastic gradient in \cref{eqn:log_lik_deriv_stochastic} requires performing two solves: $\trainK^{-1} \by$ and $\trainK^{-1} \bz$.
Using the summation formulation of CG (Eq.~\ref{eqn:cg_series}), we can write these two solves as series:
$$
\textstyle
\trainK^{-1} \by = \sum_{i=1}^N \gamma_i \bd_i, \quad
\trainK^{-1} \bz = \sum_{i=1}^N \gamma_i' \bd_i',
$$
where each CG iteration computes a new term of the summation.
By applying the Russian Roulette estimator from \cref{eqn:rr}, we obtain the following unbiased estimates:
\begin{equation}\label{eqn:rrcg-linear-solve}
\textstyle
\begin{split}
    \textstyle
    \trainK^{-1} \by &\approx {\textstyle \sum_{j=1}^{J} } (\gamma_j 
    \bd_j) / \mathbb{P}(\mathcal J \geq j),
    \quad J \sim \mathbb{P}(J)
    \\
    \textstyle
    \trainK^{-1} \bz &\approx
    {\textstyle \sum_{j=1}^{J'} }
    (\gamma_j' \bd_j') / \mathbb{P}(\mathcal J \geq j). 
    \quad J' \sim \mathbb{P}(J'),
\end{split}
\end{equation}
These unbiased solves produce an unbiased optimization gradient in \cref{eqn:log_lik_deriv_stochastic}; we refer to this approach as Russian Roulette CG ({\bf RR-CG}).
With the appropriate truncation distribution $\mathbb{P}(J)$, this estimate affords the same computational complexity of standard CG \emph{without} its bias.

We must compute two independent estimates of $\trainK^{-1} \by$ with different $J \sim \mathbb{P}(J)$ in order for the $\by^\top \trainK^{-1} \frac{ \partial \trainK }{ \partial \btheta } \trainK^{-1} \by$ term in \cref{eqn:log_lik_deriv_stochastic} to be unbiased.
Thus, the unbiased gradient requires 3 calls to RR-CG, as opposed to the 2 CG calls needed for the biased gradient.
Nevertheless, RR-CG has the same $\bigo{JN^2}$ complexity as standard CG -- and the additional solve can be computed in parallel with the others.

We can also use the Russian Roulette estimator to compute the log marginal likelihood itself, though this is not strictly necessary for gradient-based optimization.
(See appendix.)

\paragraph{Choosing the truncation distribution.}
Since the Russian Roulette estimator is unbiased for any choice of
$\mathbb{P}\left(J\right)$, we wish to choose a truncation distribution that balances
computational cost and variance\footnote{Solely minimizing variance is not appealing,
as this is achieved by $\mathbb{P}\left(J\right) = \mathbb{I}_{J=N}$ which has no
computational savings.}.
\citet{beatson2019efficient} propose the \emph{relative optimization
efficiency} (ROE) metric, which is the ratio of the expected improvement of taking an optimization step with
our gradient estimate to its computational cost.
A critical requirement of the ROE analysis is the expected rate of decay 
of our approximations in terms of the number of CG iterations $J$. 
We summarize our estimates and choices of distribution as follows:
\begin{theorem}\label{th:optimalrrcg}
  The approximation to $\logdet$ and to $\invquad$ using RR-CG decays at a rate of
  $\bigo{C^{-2J}}$. Therefore the truncation distribution that maximizes
  the ROE is $\mathbb{P}^*(J) \propto C^{-2J}$, where $C$ is a constant that depends on the conditioning of $\trainK$.
  The expected computation and variance of  $\mathbb{P}^*\left(J\right)$ is finite.
\end{theorem}

\textit{Proof sketch.}  
\citet{beatson2019efficient} show that the truncation distribution that maximizes the ROE is proportional to the rate of decay of our approximation 
divided by its computational cost. The error of CG 
decays as $\bigo{C^{-2J}}$, and the cost of each summation term is constant with respect to $J$. 
In practice, we vary the exponential decaying rate to control the expectation and variance of $\mathbb P(J)$.
To further reduce the variance of our estimator, we set a minimum number of CG iterations to be computed, as in \cite{luo2020sumo}.
See appendix for full proof. 


In practice, however, we do not have access to $C$ since
computing the conditioning of $\trainK$ is impractical. Yet, we can change the base of the exponential to $e$ and add a temperature parameter
$\lambda$ to rescale the function and control the rate of decay of the truncation distribution as a sensible alternative.
Thus, we follow a more general exponential decay distribution: 
\begin{align}\label{eqn:exp-decay-dist}
    \mathbb{P} (J) & \propto e^{-\lambda J}, \qquad J=J_{\min}, \cdots, N
\end{align}
where  $J_{\min}$ is the minimum truncation number. By varying the values of $\lambda$ and $J_{\min}$, we can control the expectation and standard deviation of $\mathbb{P}(J)$. In practice we found that having the standard deviation value between 10 and 20 achieves stable GP learning process, which can be obtained by tuning $\lambda$ between $0.05$ and $0.1$ for sufficiently large datasets (e.g. $N \geq 500$). We also noticed that the method is not sensitive to these choices of hyperparameters and that they work well across all the experiments.
The expected truncation number can be further tuned by varying  $J_{\min}$. 
We emphasize that these choices impact the speed-variance tradeoff. By setting a larger $J_{\min}$ we decrease the speed
by requiring more baseline computations but also decrease the variance (since the minimum approximations have the largest deviations from 
the ground truth).

\paragraph{Toy problem.}
In \cref{fig:debias_cg_rff} we plot the empirical mean of the RR-CG estimator using $10^4$ samples from an exponential truncation distribution.
We find that RR-CG produces unbiased estimates of the $\invquad$ and $\logdet$ terms that are indistinguishable from the exact values computed with Cholesky.
Reducing the expected truncation iteration $\mathbb{E}(J)$ (x-axis) increases the standard error of empirical means, demonstrating the speed-variance trade-off. 

\subsection{Single Sample-Truncated RFF (SS-RFF)}

Denoting $\widetilde \bK_j$ as the kernel matrix estimated by $j$ random Fourier features, we can write $\logdet$ as the following telescoping series:
\begin{align}
    \logdet &= { \log \vert \widetilde \bK_{1} \vert + \sum_{j=2}^{N/2 - 1} \left( \log \vert \widetilde \bK_j \vert - \log \vert \widetilde \bK_{j-1} \vert \right) } \\
    &+ \logdet - \log \vert \widetilde \bK_{N/2-1} \vert
    \nonumber
    \\
    &= {\textstyle \log \vert \widetilde \bK_1 \vert + \sum_{j=2}^{N/2} \Delta_j, }
    \label{eqn:rff_logdet_series} 
\end{align}
where $\Delta_j$ is defined as $\log \vert \widetilde \bK_j \vert - \log \vert \widetilde \bK_{j-1} \vert$ for all $j < N/2$, and $\Delta_{N/2}$ is defined as $\log \vert \trainK \vert - \log \vert \widetilde \bK_{N/2 - 1} \vert$. Note that each $\Delta_j$ is now a random variable, since it depends on the \textit{random} Fourier frequencies $\omega$.
Crucially, we only include $N/2$ terms in the series so that no term requires more than $\bigo{N^3}$ computation in expectation.
(For any $j > N/2$, $\widetilde \bK_{j}$ is a full-rank matrix and thus is as computationally expensive as the true $\trainK$ matrix.) 
We construct a similar telescoping series for $\invquad$.

As with \cref{eqn:rrcg-linear-solve}, we approximate the series in \cref{eqn:rff_logdet_series} with a randomized truncation estimator, this time using the Single Sample estimator \eqref{eqn:ss_estimator}:
\begin{align}
    \logdet \approx \log \vert \widetilde \bK_1 \vert +
    \Delta_{J} / \mathbb{P}\left(J\right).
    \label{eqn:ssrff_est}
\end{align}
where $J$ is drawn from the truncation distribution $\mathbb{P}(J)$ with support over $\{ 2, 3, \ldots N/2 \}$.
Note that the Single Sample estimate requires computing 3 log determinants ($\log \vert \widetilde \bK_1 \vert$, $\log \vert \widetilde \bK_{J-1} \vert$, and $\log \vert \widetilde \bK_{J} \vert$) for a total of $\bigo{J^3 + NJ^2}$ computations and $\bigo{NJ}$ memory.
This is asymptotically the same requirement as standard RFF.
The Russian Roulette estimator, on the other hand, incurs a computational cost of $\bigo{NJ^3 + J^4}$ as it requires computing ($\log \vert \widetilde \bK_1 \vert$ through $\log \vert \widetilde \bK_{J} \vert$) which quickly becomes impractical for large $J$.

A similar Single Sample estimator constructs an unbiased estimate of $\invquad$.
Backpropagating through these Single Sample estimates produces unbiased estimates of the log marginal likelihood gradient in \cref{eqn:log_lik_deriv}.

\paragraph{Choosing the truncation distribution.}
For the Single Sample estimator we do not have to optimize the ROE
since minimizing the variance of this estimator does not result in a degenerate distribution.
%
\begin{theorem}\label{th:optimalssrff}
  The truncation distribution that
  minimizes the variance of the SS-RFF estimators for $\logdet$ and $\invquad$ is
        $\mathbb{P}^{*}\left(J\right)
        \propto
        1/J$.
  The expected variance and computation of $\mathbb{P}^{*}(J)$ is finite.
\end{theorem}
\textit{Proof sketch.} The minimum variance distribution can be found by solving a constrained optimization problem. In
practice, we can further decrease the variance of our estimator by fixing a minimum value of
RFF features to be used in \cref{eqn:rff_logdet_series} and by increasing the step size
($c \in \mathbb{N}$) between the elements at each $\Delta_{j}$ =
$\log|\widetilde{\bK}_{cJ}| - \log|\widetilde{\bK}_{c(J-1)}|$. See appendix for full proof.

For the experiments we started with 500
features and also tried various step sizes $c \in \{1, 10, 100 \}$. The variance of the
estimator decreases as we increase $c$ since the probability weights will decrease in
magnitude. Yet, despite of using the optimal truncation distribution, setting a high
number of features and taking long steps $c=100$, the variance of the estimator still
requires taking several steps before converging, making SS-RFF computationally
impractical.

\paragraph{Toy problem.} Similar to RR-CG, in \cref{fig:debias_cg_rff} we plot the empirical mean of the SS-RFF estimator using $10^4$ samples.
We find that SS-RFF produces unbiased estimates of the $\invquad$ and $\logdet$ terms. However, these estimates have a higher variance when compared to the estimates of RR-CG. Reducing the expected truncation iteration $\mathbb{E}(J)$ (x-axis) increases the standard error of the empirical means, demonstrating the speed-variance trade-off.

\subsection{Analysis of the Bias-free Methods}

Randomized truncations and conjugate gradients have existed for many decades
\citep{hestenes1952methods,kahn1955use}, but have rarely been used in conjunction. 
\citet{filippone2015enabling} proposed a method closely related to our RR-CG which performs randomized early-truncation of CG iterations to obtain unbiased posterior samples of the GP covariance parameters. We differ by tackling the GP hyperparameter learning problem: we provide the first theoretical analysis of the biases incurred by both CG and RFF, and proceed to tailor unbiased estimators for each method.

To some extent, randomized truncation methods are antithetical to the original intention of CG: producing
 deterministic and nearly exact solves.
For large-scale
applications, where early truncation is necessary for computational tractability,
the ability to trade bias for variance is beneficial.
This fact is especially true in the context of GP learning, where the bias of early truncation is systematic and cannot be simply explained away as numerical imprecision.

Randomized truncation estimates are often used to estimate infinite series, where it is challenging to design truncation distributions with finite expected computation and/or variance.
We avoid such issues since CG and the telescoping RFF summations are both finite.


\section{Results}\label{sec:results}
First, we show that our bias-free methods recover nearly the same hyperparameters as exact methods (i.e. Cholesky-based optimization), whereas models that use CG and RFF converge to suboptimal hyperparameters. 
Since RR-CG and SS-RFF eliminate bias at the cost of increased variance, we then demonstrate the optimization convergence rate and draw conclusions on our methods' applicability.
Finally, we compare models optimized with RR-CG against a host of approximate GP methods across
a wide range of UCI datasets \cite{asuncion2007uci}.
All experiments are implemented in GPyTorch \citep{gardner2018gpytorch}

\begin{figure}[!t]
\vskip 0.1in
\begin{center}
\includegraphics[scale=0.25]{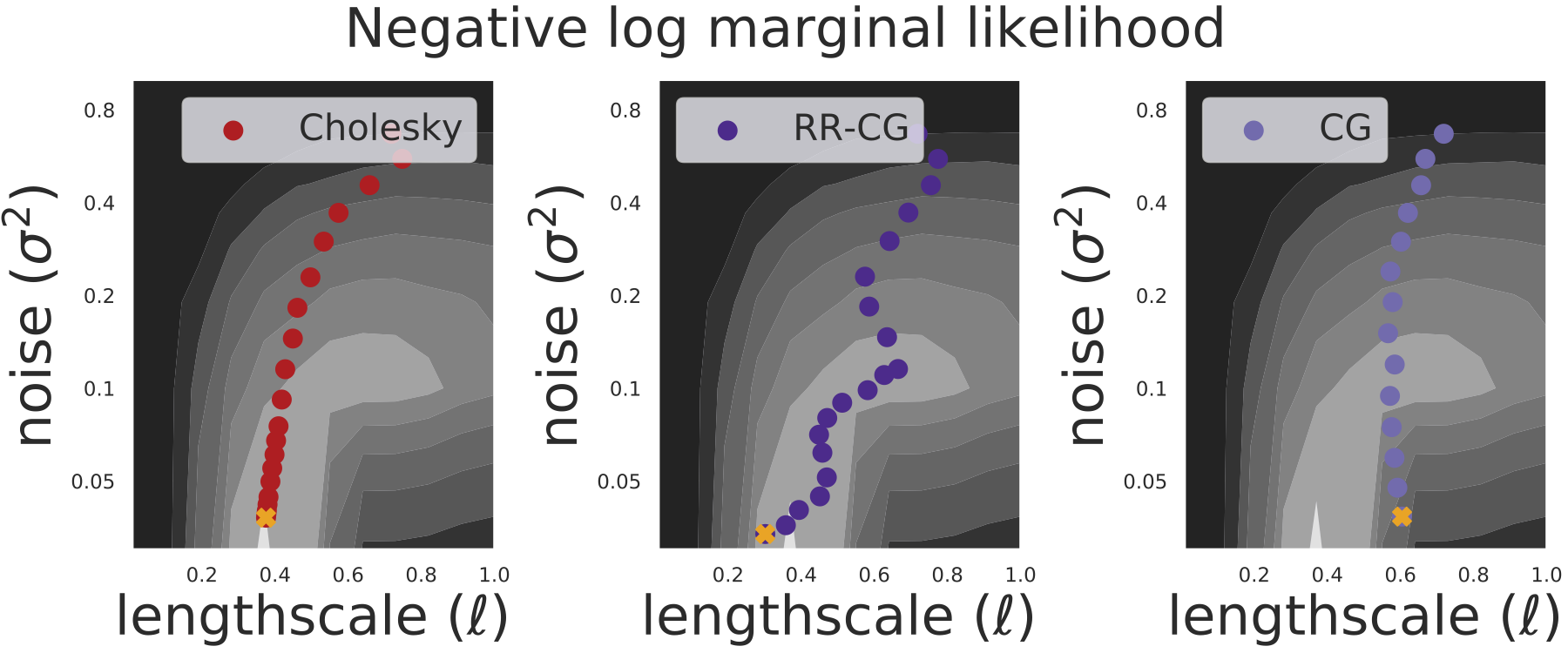}
\\
\includegraphics[scale=0.25]{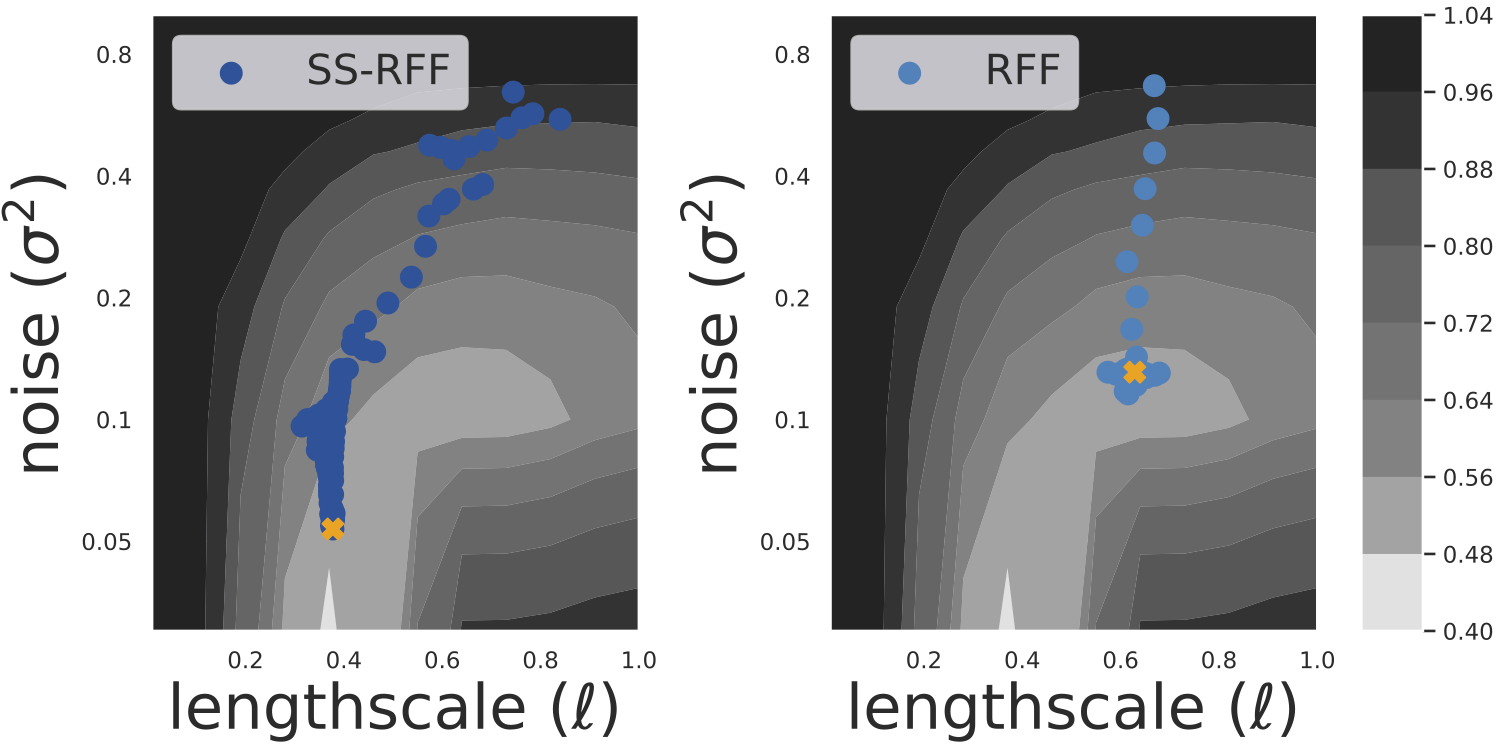}
\caption{
  Optimization landscape of a GP with two hyperparameters.
  The SS-RFF and RR-CG models converge to similar hyperparameter values that are nearly optimal, while
  the RFF and CG models converge to suboptimal solutions. In addition, the stochastic effect of the randomized truncation is visible in the trajectories of RR-CG and SS-RFF. Moreover,
  CG (and RR-CG) models truncate after $20$ iterations (in expectation); RFF (and SS-RFF) models use $700$ features (in expectation).
}
\label{fig:loss_landscapes}
\end{center}
\vskip -0.2in
\end{figure}

\paragraph{Optimization trajectories of bias-free GP.}
\label{res:loss-landscape-optim}
\cref{fig:loss_landscapes} displays the optimization landscape -- the log marginal likelihood of the PoleTele dataset -- as a function of  (RBF) kernel lengthscale $\ell$ and noise $\sigma^2$.
As  expected, an exact GP (optimized using Cholesky, \cref{fig:loss_landscapes} upper left) recovers the optimum.
Notably, the GP trained with standard CG and RFF converges to suboptimal hyperparameters (upper right/lower right).
RR-CG and SS-RFF models (trained with 20 iterations and 700 features in expectation, respectively) successfully eliminate this bias, and recover nearly the same parameters as the exact model (upper center/lower left).
These plots also show the speed-variance tradeoff of randomized truncation.
SS-RFF and RR-CG have noisy optimization trajectories due to auxiliary truncation variance.

\paragraph{Convergence of GP hyperparameter optimization.}
\label{res:loss_in_training}
\cref{fig:loss_evolution_cg,fig:loss_evolution_rff} plot the exact GP log marginal likelihood of the parameters learned by each method during optimization.
Each trajectory corresponds to a RBF-kernel GP trained on the PoleTele dataset (\cref{fig:loss_evolution_rff}) and the Bike dataset (\cref{fig:loss_evolution_cg}).

\begin{figure}[t!]
\vskip 0.1in
\begin{center}
\includegraphics[scale=0.4]{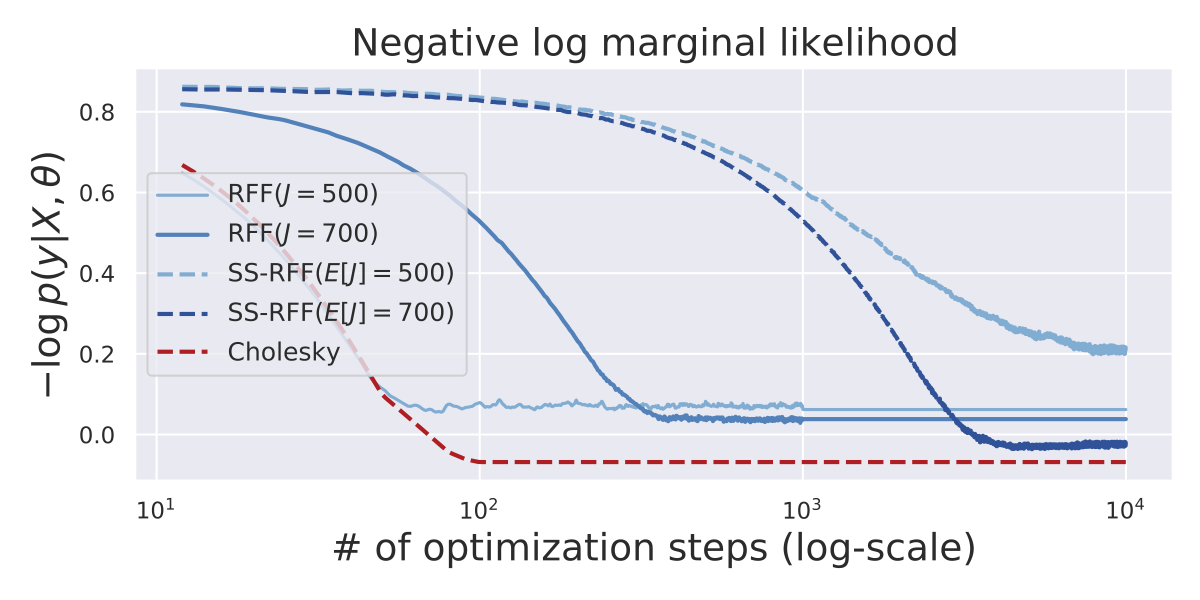}
\caption{
The GP optimization objective for models trained with RFF and SS-RFF. 
(PoleTele dataset, RBF kernel, Adam optimizer.)
RFF models converge to sub-optimal log marginal likelihoods.
SS-RFF models converge to (near) optimum values, yet require more than $100\times$ as many optimization steps.
}
\label{fig:loss_evolution_rff}
\end{center}
\vskip -0.2in
\end{figure}
\begin{figure}[t!]
\vskip 0.1in
\begin{center}
\includegraphics[scale=0.4]{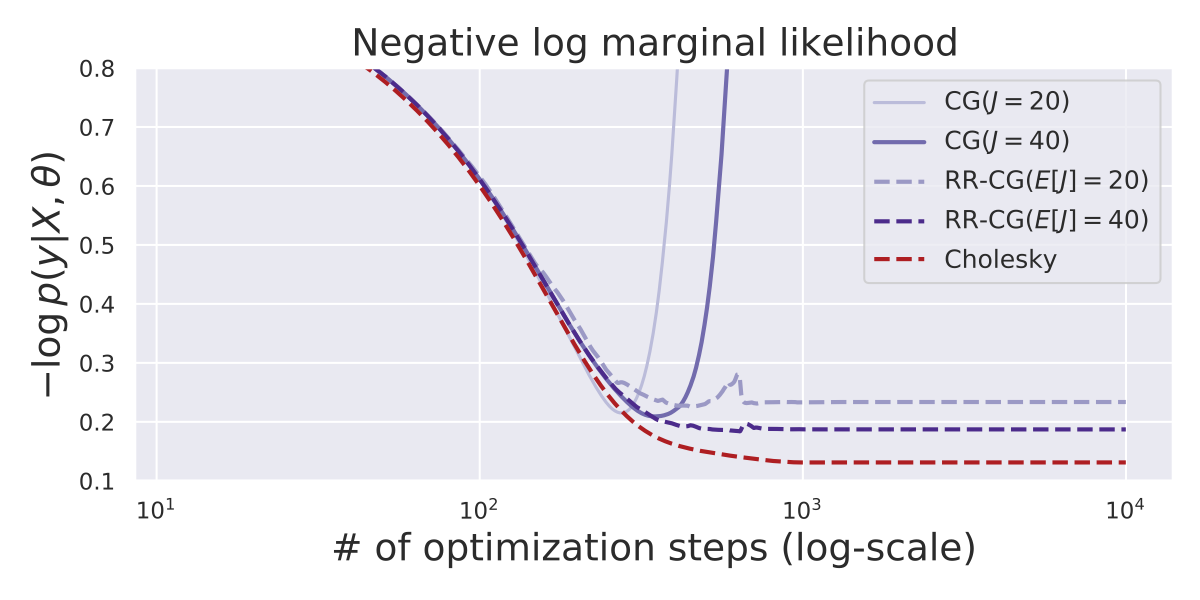}
\caption{
The GP optimization objective for models trained with CG and RR-CG. 
(Bike dataset, RBF kernel, Adam optimizer.)
RR-CG models converge to optimal solutions, while the (biased) CG models diverge.
Increasing the expected truncation of RR-CG only slightly improves optimization convergence; models converge in $<100$ steps of Adam.
}
\label{fig:loss_evolution_cg}
\end{center}
\vskip -0.2in
\end{figure}

\begin{figure*}[ht]
\vskip 0.1in
\begin{center}
\includegraphics[scale=0.45]{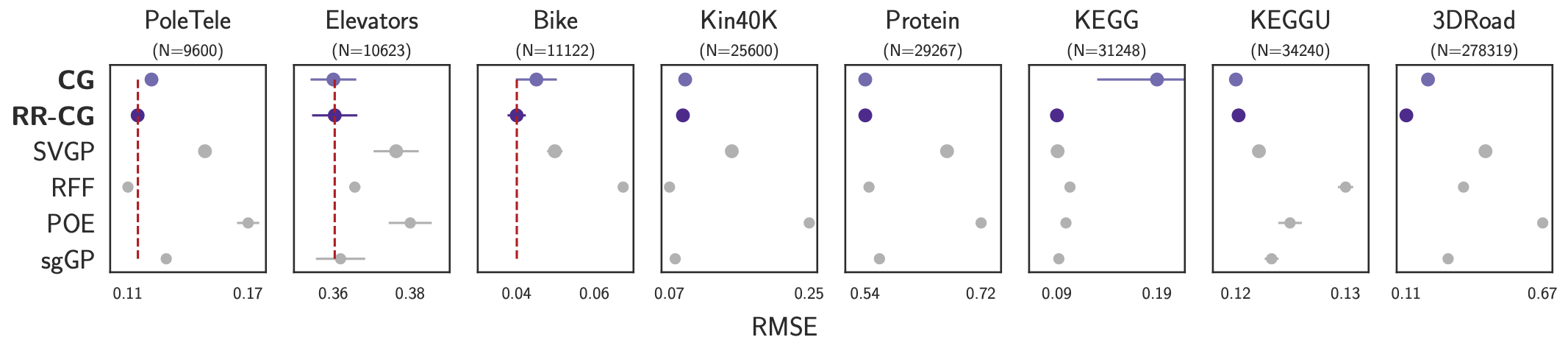}
\includegraphics[scale=0.45]{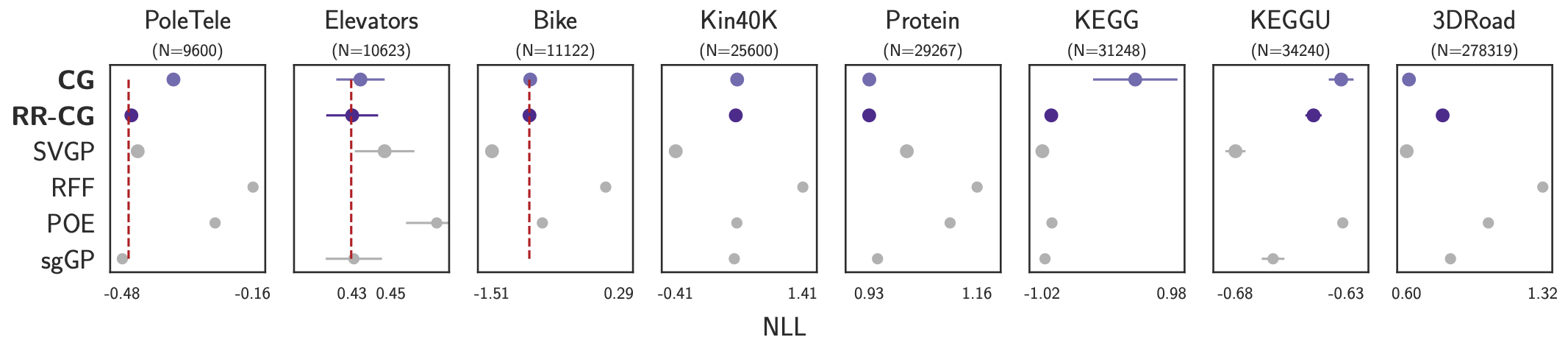}
 \caption{ Root-mean-square-error (RMSE) and negative log likelihood (NLL) of GP trained with CG (light purple), RR-CG (dark purple) and various approximate methods (grey). Dashed red lines indicates Cholesky-based GP performance (when applicable). Results are averaged over 3 dataset splits. Missing RFF and sgGP results correspond to (very high) outlier NLL / RMSE values. In almost all experiments, GP learning with RR-CG achieves similar or better performance compared to that with CG at the same computational cost.}
\label{fig:nll_rmse}
\end{center}
\vskip -0.2in
\end{figure*}

\cref{fig:loss_evolution_rff} shows that RFF models converge to solutions with worse log likelihoods, and more RFF features slow the rate of optimization.
Additionally, we see the cost of the auxiliary variance needed to debias SS-RFF: while SS-RFF models achieve better optima than their biased counterparts, they take 2-3 orders of magnitude longer to converge, despite using a truncation distribution that minimizes variance.
We thus conclude that SS-RFF has too much variance to be practical for GP hyperparameter learning.

\cref{fig:loss_evolution_cg} on the other hand shows that RR-CG is minimally affected by its auxiliary variance. The GP trained with RR-CG converges in roughly $100$ iterations, nearly matching Cholesky-based optimization.
Decreasing the expected truncation value from $\mathbb{E}[J] = 40$ to $20$ slightly slows this convergence.
We note that the bias induced by standard CG can be especially detrimental to GP learning.
On this dataset, the biased models deviate from their Cholesky counterparts and eventually diverge away from the optimum.

%

\paragraph{Predictive performance of bias-free GP.}
Lastly, we compare the predictive performance of GPs that use RR-CG, CG, and Cholesky for hyperparameter optimization.
We emphasize that the RR-CG and CG methods only make use of early truncation approximations during training.
At test time, we compute the predictive posterior by running CG to a tolerance of $\leq 10^{-4}$, which we believe can be considered ``exact'' for practical intents and purposes.
Additionally, we include four other (biased) scalable GP approximations methods as baselines: {\bf RFF}, Stochastic Variational Gaussian Processes ({\bf SVGP}) \citep{hensman2013gaussian}, generalized Product of Expert Gaussian Processes ({\bf POE}) \citep{cao2014generalized,deisenroth2015distributed}, and stochastic gradient-based Gaussian Processes ({\bf sgGP}) \citep{chen2020stochastic}.
We note that the RFF, SVGP, and sgGP methods introduce \emph{both} bias and variance, as these methods rely on randomization and approximation.

We use CG with $J=100$ iterations, and RR-CG with $\mathbb{E}[J]=100$ expected iterations; both methods use the preconditioner of \citet{gardner2018gpytorch}.
All RFF models use $700$ random features.
For SVGP, we use $1{,}024$ inducing points and minibatches of size $1{,}024$ as in \cite{wang2019exact}.
The POE models are comprised of GP experts that are each trained on $1{,}024$ data point subsets.
For sgGP, the subsampled datasets are constructed by selecting a random point $\bx, y$ and its 15 nearest neighbors as in \cite{chen2020stochastic}.
Each dataset is randomly split to  64\% training,  16\% validation and 20\% testing sets.
All kernels are RBF with a separate lengthscale per dimension. See appendix for more details. 

We report prediction accuracy (RMSE) and negative log likelihood (NLL) in \cref{fig:nll_rmse} (see appendix for full tables on predictive performance and training time).
We make two key observations: \emph{(i)} RR-CG meaningfully debiases CG. When the bias of CG is not detrimental to optimization (e.g. CG with 100 iterations is close to convergence for the Elevators dataset), RR-CG has similar performance. However, when the CG bias is more significant (e.g. the KEGG dataset), the bias-free RR-CG improves the GP predictive RMSE and NLL. We also include a figure displaying the predictive performance of RR-CG and CG with increasing number of (expected) CG iterations in appendix.  \emph{(ii)} RR-CG recovers the same optimum as the ``ground-truth'' method (i.e. Cholesky) does, as indicated by the red-dashed line in \cref{fig:nll_rmse}. This result provides additional evidence that RR-CG achieves unbiased optimization. 

While RR-CG obtains the lowest RMSE on all but 2 datasets, we note that the (biased) GP approximations sometimes achieve lower NLL. For example, SVGP has a lower NLL than that of RR-CG on the Bike dataset, despite having a higher RMSE.
We emphasize that this is not a failing of RR-CG inference.
The SVGP NLL is even better than that of the exact (Cholesky) GP, suggesting a potential model misspecification for this particular dataset. Since SVGP overestimates the observational noise $\sigma^2$ \citep{bauer2017understanding}, it may obtain a better NLL when outliers are abundant.
Though we cannot compare against the Cholesky posterior on larger datasets, we hypothesize that the NLL/RMSE discrepancy on these datasets is due to a similar modeling issue.

\section{Conclusion}
We prove that CG and RFF introduce systematic biases to the GP log marginal likelihood objective:
CG-based training will favor underfitting models, while RFF-based training will promote overfitting.
Modifying these methods with randomized truncation converts these biases into variance, enabling unbiased stochastic optimization.
Our results show that this bias-to-variance exchange indeed constitutes a trade-off.
The convergence of SS-RFF is impractically slow, likely due to the truncation variance needed to eliminate RFF's slowly-decaying bias.
However, for CG-based training, we find that variance is almost always preferable to bias.
Models trained with RR-CG achieve better performance than those trained with standard CG, and tend to recover the hyperparameters learned with exact methods.
Though models trained with CG do not always exhibit noticeable bias, RR-CG's negligible computational overhead is justifiable to counteract cases where the bias is significant.

We reported experiments with at most 300K observations for our methods and baselines, which is substantial for GPs. We emphasize that RR-CG can be extended to datasets with over one million data points as in \citet{wang2019exact}. However, the computational cost is much higher, requiring multiple GPUs for training and testing.

We note that the RR-CG algorithm is not limited to GP applications.
Future work should explore applying RR-CG to other optimization problems with large-scale solves.


\section*{Acknowledgements}
This work was supported by the Simons Foundation, McKnight Foundation, the Grossman Center, and the Gatsby
Charitable Trust.

\bibliography{citations}

\begin{thebibliography}{38}
\providecommand{\natexlab}[1]{#1}
\providecommand{\url}[1]{\texttt{#1}}
\expandafter\ifx\csname urlstyle\endcsname\relax
  \providecommand{\doi}[1]{doi: #1}\else
  \providecommand{\doi}{doi: \begingroup \urlstyle{rm}\Url}\fi

\bibitem[Angelova(2012)]{angelova2012moments}
Angelova, J.~A.
\newblock On moments of samples mean and variance.
\newblock In \emph{Int. J. Pure Appl. Math}, pp.\  79:67--85, 2012.

\bibitem[Asuncion \& Newman(2007)Asuncion and Newman]{asuncion2007uci}
Asuncion, A. and Newman, D.
\newblock Uci machine learning repository, 2007.

\bibitem[Bauer et~al.(2016)Bauer, van~der Wilk, and
  Rasmussen]{bauer2017understanding}
Bauer, M., van~der Wilk, M., and Rasmussen, C.~E.
\newblock Understanding probabilistic sparse gaussian process approximations.
\newblock In \emph{Advances in Neural Information Processing Systems}, 2016.

\bibitem[Beatson \& Adams(2019)Beatson and Adams]{beatson2019efficient}
Beatson, A. and Adams, R.~P.
\newblock Efficient optimization of loops and limits with randomized
  telescoping sums.
\newblock In \emph{International Conference on Machine Learning}, 2019.

\bibitem[Bochner et~al.(1959)]{bochner1959lectures}
Bochner, S. et~al.
\newblock \emph{Lectures on Fourier integrals}, volume~42.
\newblock Princeton University Press, 1959.

\bibitem[Burt et~al.(2019)Burt, Rasmussen, and Van Der~Wilk]{burt2019rates}
Burt, D., Rasmussen, C.~E., and Van Der~Wilk, M.
\newblock Rates of convergence for sparse variational gaussian process
  regression.
\newblock In \emph{International Conference on Machine Learning}, pp.\
  862--871, 2019.

\bibitem[Cao \& Fleet(2014)Cao and Fleet]{cao2014generalized}
Cao, Y. and Fleet, D.~J.
\newblock Generalized product of experts for automatic and principled fusion of
  gaussian process predictions.
\newblock \emph{arXiv preprint arXiv:1410.7827}, 2014.

\bibitem[Chen et~al.(2020)Chen, Zheng, Al~Kontar, and
  Raskutti]{chen2020stochastic}
Chen, H., Zheng, L., Al~Kontar, R., and Raskutti, G.
\newblock Stochastic gradient descent in correlated settings: A study on
  gaussian processes.
\newblock \emph{Advances in Neural Information Processing Systems}, 33, 2020.

\bibitem[Chen et~al.(2019)Chen, Behrmann, Duvenaud, and
  Jacobsen]{chen2020residual}
Chen, R. T.~Q., Behrmann, J., Duvenaud, D., and Jacobsen, J.-H.
\newblock Residual flows for invertible generative modeling.
\newblock \emph{Advances in Neural Information Processing Systems}, 2019.

\bibitem[Cunningham et~al.(2008)Cunningham, Shenoy, and
  Sahani]{cunningham2008fast}
Cunningham, J.~P., Shenoy, K.~V., and Sahani, M.
\newblock Fast gaussian process methods for point process intensity estimation.
\newblock In \emph{International Conference on Machine learning}, pp.\
  192--199, 2008.

\bibitem[Cutajar et~al.(2016)Cutajar, Osborne, Cunningham, and
  Filippone]{cutajar2016preconditioning}
Cutajar, K., Osborne, M.~A., Cunningham, J.~P., and Filippone, M.
\newblock Preconditioning kernel matrices.
\newblock In \emph{International Conference on Machine Learning}, 2016.

\bibitem[Datta et~al.(2016)Datta, Banerjee, Finley, and
  Gelfand]{datta2016hierarchical}
Datta, A., Banerjee, S., Finley, A.~O., and Gelfand, A.~E.
\newblock Hierarchical nearest-neighbor gaussian process models for large
  geostatistical datasets.
\newblock \emph{Journal of the American Statistical Association}, 111\penalty0
  (514):\penalty0 800--812, 2016.

\bibitem[Deisenroth \& Ng(2015)Deisenroth and Ng]{deisenroth2015distributed}
Deisenroth, M. and Ng, J.~W.
\newblock Distributed gaussian processes.
\newblock In \emph{International Conference on Machine Learning}, pp.\
  1481--1490. PMLR, 2015.

\bibitem[Dong et~al.(2017)Dong, Eriksson, Nickisch, Bindel, and
  Wilson]{dong2017scalable}
Dong, K., Eriksson, D., Nickisch, H., Bindel, D., and Wilson, A.~G.
\newblock Scalable log determinants for gaussian process kernel learning.
\newblock In \emph{Advances in Neural Information Processing Systems}, pp.\
  6327--6337, 2017.

\bibitem[Filippone \& Engler(2015)Filippone and Engler]{filippone2015enabling}
Filippone, M. and Engler, R.
\newblock Enabling scalable stochastic gradient-based inference for gaussian
  processes by employing the unbiased linear system solver (ulisse).
\newblock In \emph{International Conference on Machine Learning}, pp.\
  1015--1024. PMLR, 2015.

\bibitem[Gardner et~al.(2018)Gardner, Pleiss, Weinberger, Bindel, and
  Wilson]{gardner2018gpytorch}
Gardner, J., Pleiss, G., Weinberger, K.~Q., Bindel, D., and Wilson, A.~G.
\newblock Gpytorch: Blackbox matrix-matrix gaussian process inference with gpu
  acceleration.
\newblock In \emph{Advances in Neural Information Processing Systems}, pp.\
  7576--7586, 2018.

\bibitem[Golub \& Meurant(2009)Golub and Meurant]{golub2009matrices}
Golub, G.~H. and Meurant, G.
\newblock \emph{Matrices, moments and quadrature with applications}, volume~30.
\newblock Princeton University Press, 2009.

\bibitem[Golub \& Van~Loan(2012)Golub and Van~Loan]{golub2012matrix}
Golub, G.~H. and Van~Loan, C.~F.
\newblock \emph{Matrix Computations}.
\newblock The Johns Hopkins University Press, 4th Edition, 2012.

\bibitem[Hensman et~al.(2013)Hensman, Fusi, and Lawrence]{hensman2013gaussian}
Hensman, J., Fusi, N., and Lawrence, N.~D.
\newblock {G}aussian processes for big data.
\newblock In \emph{Uncertainty in Artificial Intelligence}, 2013.

\bibitem[Hestenes et~al.(1952)Hestenes, Stiefel, et~al.]{hestenes1952methods}
Hestenes, M.~R., Stiefel, E., et~al.
\newblock Methods of conjugate gradients for solving linear systems.
\newblock \emph{Journal of Research of the National Bureau of Standards},
  49\penalty0 (1), 1952.

\bibitem[Hutchinson(1989)]{hutchinson1989stochastic}
Hutchinson, M.~F.
\newblock A stochastic estimator of the trace of the influence matrix for
  laplacian smoothing splines.
\newblock \emph{Communications in Statistics-Simulation and Computation},
  18\penalty0 (3):\penalty0 1059--1076, 1989.

\bibitem[Kahn(1955)]{kahn1955use}
Kahn, H.
\newblock Use of different monte carlo sampling techniques.
\newblock \emph{Rand Corporation}, 1955.

\bibitem[Katzfuss et~al.(2021)Katzfuss, Guinness, et~al.]{katzfuss2021general}
Katzfuss, M., Guinness, J., et~al.
\newblock A general framework for vecchia approximations of gaussian processes.
\newblock \emph{Statistical Science}, 36\penalty0 (1):\penalty0 124--141, 2021.

\bibitem[Lanczos(1950)]{lanczos1950iteration}
Lanczos, C.
\newblock An iteration method for the solution of the eigenvalue problem of
  linear differential and integral operators1.
\newblock \emph{Journal of Research of the National Bureau of Standards},
  45\penalty0 (4), 1950.

\bibitem[Loper et~al.(2020)Loper, Blei, Cunningham, and
  Paninski]{loper2020general}
Loper, J., Blei, D., Cunningham, J.~P., and Paninski, L.
\newblock General linear-time inference for gaussian processes on one
  dimension.
\newblock \emph{arXiv preprint arXiv:2003.05554}, 2020.

\bibitem[Luo et~al.(2020)Luo, Beatson, Norouzi, Zhu, Duvenaud, Adams, and
  Chen]{luo2020sumo}
Luo, Y., Beatson, A., Norouzi, M., Zhu, J., Duvenaud, D., Adams, R.~P., and
  Chen, R. T.~Q.
\newblock Sumo: Unbiased estimation of log marginal probability for latent
  variable models.
\newblock In \emph{International Conference on Learned Representations}, 2020.

\bibitem[Lyne et~al.(2015)Lyne, Girolami, Atchad{\'e}, Strathmann, Simpson,
  et~al.]{lyne2015russian}
Lyne, A.-M., Girolami, M., Atchad{\'e}, Y., Strathmann, H., Simpson, D., et~al.
\newblock On russian roulette estimates for bayesian inference with
  doubly-intractable likelihoods.
\newblock \emph{Statistical science}, 30\penalty0 (4):\penalty0 443--467, 2015.

\bibitem[Mutn{\`y} \& Krause(2018)Mutn{\`y} and Krause]{mutny2019efficient}
Mutn{\`y}, M. and Krause, A.
\newblock Efficient high dimensional bayesian optimization with additivity and
  quadrature fourier features.
\newblock \emph{Advances in Neural Information Processing Systems}, pp.\
  9005--9016, 2018.

\bibitem[Nowozin(2018)]{nowozin2018debiasing}
Nowozin, S.
\newblock Debiasing evidence approximations: On importance-weighted
  autoencoders and jackknife variational inference.
\newblock In \emph{International Conference on Learned Representations}, 2018.

\bibitem[Oktay et~al.(2020)Oktay, McGreivy, Aduol, Beatson, and
  Adams]{oktay2020randomized}
Oktay, D., McGreivy, N., Aduol, J., Beatson, A., and Adams, R.~P.
\newblock Randomized automatic differentiation.
\newblock \emph{arXiv preprint arXiv:2007.10412}, 2020.

\bibitem[Pleiss et~al.(2018)Pleiss, Gardner, Weinberger, and
  Wilson]{pleiss2018constant}
Pleiss, G., Gardner, J.~R., Weinberger, K.~Q., and Wilson, A.~G.
\newblock Constant-time predictive distributions for gaussian processes.
\newblock In \emph{International Conference on Machine learning}, 2018.

\bibitem[Rahimi \& Recht(2008)Rahimi and Recht]{rahimi2008random}
Rahimi, A. and Recht, B.
\newblock Random features for large-scale kernel machines.
\newblock In \emph{Advances in Neural Information Processing Systems}, pp.\
  1177--1184, 2008.

\bibitem[Snelson \& Ghahramani(2006)Snelson and Ghahramani]{snelson2006sparse}
Snelson, E. and Ghahramani, Z.
\newblock Sparse {G}aussian processes using pseudo-inputs.
\newblock In \emph{Advances in Neural Information Processing Systems}, 2006.

\bibitem[Sutherland \& Schneider(2015)Sutherland and
  Schneider]{sutherland2015error}
Sutherland, D.~J. and Schneider, J.~G.
\newblock On the error of random fourier feaures.
\newblock \emph{Conference on Uncertainty in Artificial Intelligence}, 2015.

\bibitem[Titsias(2009)]{titsias2009variational}
Titsias, M.~K.
\newblock Variational learning of inducing variables in sparse {G}aussian
  processes.
\newblock In \emph{International Conference on Artificial Intelligence and
  Statistics}, pp.\  567--574, 2009.

\bibitem[Ubaru et~al.(2017)Ubaru, Chen, and Saad]{ubaru2017fast}
Ubaru, S., Chen, J., and Saad, Y.
\newblock Fast estimation of tr(f(a)) via stochastic lanczos quadrature.
\newblock \emph{SIAM Journal on Matrix Analysis and Applications}, 38\penalty0
  (4):\penalty0 1075--1099, 2017.

\bibitem[Wang et~al.(2019)Wang, Pleiss, Gardner, Tyree, Weinberger, and
  Wilson]{wang2019exact}
Wang, K., Pleiss, G., Gardner, J., Tyree, S., Weinberger, K.~Q., and Wilson,
  A.~G.
\newblock Exact gaussian processes on a million data points.
\newblock In \emph{Advances in Neural Information Processing Systems}, pp.\
  14648--14659, 2019.

\bibitem[Wilson et~al.(2020)Wilson, Borovitskiy, Terenin, Mostowsky, and
  Deisenroth]{wilson2020efficiently}
Wilson, J., Borovitskiy, V., Terenin, A., Mostowsky, P., and Deisenroth, M.
\newblock Efficiently sampling functions from gaussian process posteriors.
\newblock In \emph{International Conference on Machine Learning}, pp.\
  10292--10302, 2020.

\end{thebibliography}
\bibliographystyle{icml2021}

\clearpage
\appendix

\makeatletter
	\setcounter{table}{0}
	\renewcommand{\thetable}{S\arabic{table}}%
	\setcounter{figure}{0}
	\renewcommand{\thefigure}{S\arabic{figure}}%
    \setcounter{equation}{0}
    \renewcommand\theequation{S\arabic{equation}}
\makeatother

\section{Proof of Biases for CG and RFF}
\subsection{Proof of Theorem~\ref{CG_bias_theorem}}
\begin{proof}
To prove the bias of the CG log marginal likelihood terms, we rely on a connection between conjugate gradients, the Lanczos algorithm \cite{lanczos1950iteration}, and Gauss quadrature.
In particular, we demonstrate that $u_J$ and $v_J$ -- CG's estimates for $\invquad$ and $\logdet$ terms -- are equivalent to Gauss quadrature approximations of two Riemann-Stieljes integrals.
We then use quadrature error analysis to prove the biases of these terms.

\subparagraph{Expressing $\invquad$ and $\bz^\top (\log \trainK) \bz$ as Riemann-Stieljes integrals.}
First, we note that $\invquad$ and $\bz^\top (\log \trainK) \bz$ (our stochastic trace estimate of $\logdet$) can both be expressed by the quadratic form $\bw^\top f(\bA) \bw$, where $f (\bA)$ denotes a matrix function of the matrix $\bA$, and $\bw$ is a vector.
Letting $\bA = \bP \bLambda \bP^\top$ be the eigendecomposition of $\bA$, we can write this quadratic form as
\begin{align*}
    \bw^\top f(\bA) \bw &= \bw^\top \bP f(\bLambda) \bP^\top \bw = \Vert \bw \Vert^2 \sum_{i=1}^N f(\lambda_i) \mu_i^2, 
\end{align*}
where $\lambda_{\min} \leq \lambda_i \leq \lambda_{\max}$ are the diagonal elements of $\bLambda$ (i.e. the eigenvalues) -- ordered from smallest to largest, and $\mu_i$ are the components of $\bP^\top \bw / \| \bw \|$. The summation in the above equation can be expressed as a Riemann-Stieltjes integral:
\begin{align}
    I[f] & \defeq  \Vert \bw \Vert^2 \sum_{i=1}^N f(\lambda_i) \mu_i^2
    \nonumber \\
    &=  \Vert \bw \Vert^2 \int_{\lambda_{\min}}^{\lambda_{\max}} f(t) d\mu_{\bA}(t),
    \label{eqn:rs_integral}
\end{align} 
where the measure $\mu_{\bA}(t)$ is a piecewise constant function 
\begin{align}
    \mu_{\bA}(t) &= \begin{cases}
    0, \qquad &\textrm{if } t < \lambda_{\min}, \\
    \sum_{j=1}^i \mu_j^2, \qquad & \textrm{if } \lambda_i \leq t < \lambda_{i-1} \\
    \sum_{j=1}^N \mu_j^2, \qquad & \textrm{if } \lambda_{\max} \leq t. 
    \end{cases}
    \label{eqn:rs_measure}
\end{align}
See \citep{golub2009matrices} for more details.

\subparagraph{The Lanczos algorithm approximates these integrals with Gauss quadrature.}
The Lanczos algorithm \cite{lanczos1950iteration}, which is briefly described in \cref{subsec:CG_bias}, iteratively expresses a symmetric matrix $\bA$ via the partial tridiagonalization $\bT^{(J)}_\bw = \bQ^{(J)\top}_\bw \bA \bQ^{(J)}_\bw$.
$\bQ^{(J)}_\bw \in \reals^{N \times J}$ is a orthonormal matrix with $\bw / \Vert \bw \Vert$ as its first column, and $\bT^{(J)}_\bw$ is a $J \times J$ tridiagonal matrix.
Briefly, the columns of $\bQ^{(J)}$ matrices are computed by performing Gram-Schmidt orthogonalization on the Krylov subspace $[ \bw, \bA \bw, \bA^2 \bw, \ldots, \bA^{J-1} \bw ]$, and storing the orthogonalization constants in $\bT^{(J)}_\bw$.

The Lanczos algorithm is commonly used to estimate quadratic forms:
\begin{align}
    \bw^\top f(\bA) \bw
    &\approx
    \bw^\top \bQ^{(J)})_\bw f \left( \bT^{(J)}_\bw \right) \bQ^{(J)\top})_\bw \bw
    \nonumber
    \\
    &=
    \Vert \bw \Vert^2 \be^\top_1 f \left( \bT^{(J)}_\bw \right) \be_1,
    \label{eqn:lanczos_quadrature}
\end{align}
where $\be_1$ is the first unit vector.
Note that \cref{eqn:lanczos_quadrature} holds because the columns of $\bQ^{(J)}_\bw$ are orthonormal.

There is a well-established connection between \cref{eqn:lanczos_quadrature} and numeric quadrature \citep[e.g.][]{golub2009matrices}.
More specifically, \cref{eqn:lanczos_quadrature} is exactly equivalent to a $J$-term Gauss quadrature rule applied to the Riemann-Stieltjes integral in \cref{eqn:rs_integral}.
We can thus analyze Lanczos estimates of $\bw^\top f(\bA) \bw$ using standard Gauss quadrature error analysis.


\subparagraph{Equivalence between CG and the Lanczos algorithm.}
We will now show an equivalence between our estimates $u_J \approx \invquad$ and $v_J \approx \bz^\top (\log \trainK) \bz$ and Lanczos algorithm approximations.
Note that we have already established $v_J \approx \bz^\top (\log \trainK) \bz$ as a Lanczos algorithm approximation in \cref{eqn:logdet_estimate}:
$$
\bz^\top = \Vert \bz \Vert^2 \be_1^\top (\log \bT^{(J)}_\bz) \be_1.
$$
For $u_J \approx \invquad$, we exploit a connection between CG and Lanczos \citep[][Ch. 11.3]{golub2012matrix}:
\begin{equation}
    \sum_{j=1}^J \gamma_j \bd_j = \Vert \by \Vert \bQ_\by^{(J)} (\bT_{\by}^{(J)})^{-1}  \be_1,
    \label{eqn:cg-lanczos}
\end{equation}
where $\sum_{j=1}^J \gamma_j \bd_j$ (see Eq.~\ref{eqn:cg_series}) is the $J^\text{th}$ CG approximation to $\trainK^{-1} \by$.
Multiplying \cref{eqn:cg-lanczos} by $\by^\top$, we have 
\begin{align*}
    u_J  &= \by^\top \left( \sum_{j=1}^J \gamma_j \bd_j \right)
    \\
    &= \| \by\| \by^\top \bQ_\by^{(J)} (\bT_{\by}^{(J)})^{-1} \be_1
    \\
    &= \| \by \|^2 \be_1^\top (\bT_\by^{(J)})^{-1} \be. 
\end{align*}
Therefore, the CG approximations of $\invquad$ and $\bz^\top (\log \trainK) \bz$ are both Lanczos approximations.
Putting all the pieces together, we have just shown that $u_J$ and $v_J$ are equivalent to $J$-term Gauss quadrature rules applied to the following integrals:
\begin{align}
    \by^\top \trainK^{-1} \by &= \Vert \by \Vert^2 \int_{\lambda_{\min}}^{\lambda_{\max}} t^{-1} \intd\mu_{\trainK}(t),
    \label{eqn:rs_integral_invquad}
    \\
    \bz^\top \left( \log  \trainK^{-1} \right) \bz &= \Vert \bz \Vert^2 \int_{\lambda_{\min}}^{\lambda_{\max}} \log (t) \intd\mu_{\trainK}(t),
    \label{eqn:rs_integral_logdet}
\end{align}
where the measure $\mu_{\trainK}(t)$ is defined by \cref{eqn:rs_measure}.

\paragraph{Applying Gauss quadrature error analysis to $u_J$ and $v_J$.}
To analyze the bias of teh CG estimates, we make use of standard error Guass quadrature error analysis.
If $L_G^{(J)}[f]$ is the $J$-term Gaussian quadrature approximation of the Riemann-Stieltjes integral $I[f]$ (see Eq.~\ref{eqn:rs_integral}), the error can be exactly expressed as:
\begin{align}
    I[f] - L_G^{(J)} [f] &= (\gamma_1 \cdots \gamma_{J-1})^2 \frac{f^{(2J)} (\eta)}{(2J) \!},  
\end{align}
for some $\eta \in [\lambda_{\min}, \lambda_{\max}]$, where $f^{(2J)}$ is the $2J^\text{th}$ derivative of $f$ and $\{\gamma_i\}$ are some quantities that depend on the spectrum of the matrix $\trainK$ \citep[][Ch. 6]{golub2009matrices}.

Turning back to $u_J$, the corresponding function $f(t) = t^{-1}$ has even derivatives of the form $f^{(2J)} = (2J)! \: t^{-(2J+1)}  >0, \forall J, \forall t \in (\lambda_{\min}, \lambda_{\max})$.
Replacing $I[f]$ with the integral defined by \cref{eqn:rs_integral_invquad}, we have that $\invquad - u_J \geq 0$, which proves that CG underestimates $\invquad$.
Similarly for $v_J$, the corresponding $f(t) = \log t$ has even derivatives of the form $f^{(2J)} = -(2J-1)! \: t^{-2J} < 0, \forall J, \forall t \in (\lambda_{\min}, \lambda_{\max})$.
Replacing $I[f]$ with the integral defined by \cref{eqn:rs_integral_logdet}, we have that $\bz^\top ( \log \trainK ) \bz - v_J \leq 0$.

\subparagraph{The convergence rates} of $u_J$ and $v_J$ follow from Eq.~4.4 of \citet{ubaru2017fast}.
Let $\rho = (\sqrt \kappa + 1)/(\sqrt \kappa - 1)$, where $\kappa$ is the condition number of $\trainK$ (i.e. the ratio of its maximum and minimum singular values). Then
\begin{align*}
  \Vert \by^\top \trainK^{-1} \by - u_J \Vert &\leq C\rho^{-2J}, \\
  \Vert \bz^\top (\log \trainK) \bz - v_J \Vert &\leq C\rho^{-2J},
\end{align*}
where $C$ is a constant that depends on the extremal eigenvalues of $\trainK$.
\end{proof}

\subsection{Proof of Theorem~\ref{RFF_bias_theorem}}

\begin{proof}
The inequalities are a result of applying Jensen's
inequality to the inverse of a positive definite matrix (which is a convex function) and
the log determinant of a positive definite matrix (which is a concave function).
  For example, for the $\invquad$ term we have that 
  \begin{equation*}
      \begin{split}
        \Evover{\mathbb{P}(\bomega)}{\by^{\top} \widetilde{\bK}_{J}^{-1} \by}
        &=
        \by^{\top} \Evover{\mathbb{P}(\bomega)}{\widetilde{\bK}_{J}^{-1}} \by
        \\
        &\geq
        \by^{\top} \left(\Evover{\mathbb{P}(\bomega)}{\widetilde{\bK}_{J}}\right)^{-1} \by
        \\
        &=
        \invquad
        \\
      \end{split}
  \end{equation*}
  A similar procedure, but in the opposite direction applies to the $\logdet$. 
  
  To estimate the rate of decay of the biases, we rely on two keys ideas. The first idea is to
translate the central moments of the kernel matrix being approximated by the random
Fourier features into the central moments of these features. This strategy 
resembles the analysis in \cite{nowozin2018debiasing} which uses
\citep[][Theorem~1]{angelova2012moments}. The second idea is to use an approximation to two matrix
functions. For the inverse function we use a Neumann series and for the logarithm we use
a Taylor series. These series require that the eigenvalues of the approximation residual ($\trainK - \widetilde K_J)$ are less than 1.
We will make this assumption as we know that for a large enough $J$, the kernel approximation will be close to $\trainK$. 
Below is a formal argument.
  
  For a fixed $\bomega$ we can write
  \begin{equation}\label{eq:trick}
      \begin{split}
          &\by^{\top} \widetilde{\bK}_{J}^{-1} \by
          =
          \by^{\top} \left(\trainK - \trainK + \widetilde{\bK}_{J}\right)^{-1} \by
          \\
          &=
          \by^{\top} \trainK^{-1/2}
          \left(\bI- \bI + \trainK^{-1/2}\widetilde{\bK}_{J}\trainK^{-1/2}\right)^{-1}
          \trainK^{-1/2} \by
          \\
          &=
          \by^{\top} \trainK^{-1/2}
          \left(\bI- \left(\bI - \trainK^{-1/2}\widetilde{\bK}_{J}\trainK^{-1/2}\right)\right)^{-1}
          \trainK^{-1/2} \by
      \end{split}
  \end{equation}
  this last form allow us to use a Neumann series to expand the inner inverse matrix.
  Hence, we have that
  \begin{equation*}
      \begin{split}
         &\left(\bI- \left(\bI - \trainK^{-1/2}\widetilde{\bK}_{J}\trainK^{-1/2}\right)\right)^{-1}
         \\
         &=
        \sum_{t=0}^{\infty}
        \left(\bI - \trainK^{-1/2}\widetilde{\bK}_{J}\trainK^{-1/2}\right)^{t}
      \end{split}
  \end{equation*}
  Combining this with Eq. (\ref{eq:trick}), we have that
  \begin{equation} \label{eq:6}
      \begin{split}
          &\Evover{\mathbb{P}(\bomega)}{\by^{\top} \widetilde{\bK}_{J}^{-1} \by}
          =
          \by^{\top} \trainK^{-1} \by
          \\
          &+
          \sum_{t=2}^{\infty}
          \by^{\top} \trainK^{-1/2}
          \Evover{\mathbb{P}(\bomega)}
          {\left(\bI - \trainK^{-1/2}\widetilde{\bK}_{J}\trainK^{-1/2}\right)^{t}}
          \trainK^{-1/2} \by
      \end{split}
  \end{equation}
  where the first term of the series ($t=0$) is the data-fit term being approximated and
  the second term of the series cancels out ($t=1$). Following the analysis in
  \cite{nowozin2018debiasing} we translate the central moments of the random variable
  $\trainK^{-1/2}\widetilde{\bK}_{J}\trainK^{-1/2}$ to the central moments of
  $\trainK^{-1/2} \left(\bphi\left(\bomega\right) \bphi\left(\bomega\right)^{\top}
  + \sigma^{2} \bI\right)
  \trainK^{-1/2}$ denoted as $\mu_{i}$ for $i \geq 2$. Since the term $(t=2)$ will
  dominate, we will only focus on it (as the others would be of a higher order in $J$,
  see \cite{angelova2012moments}).
  We have that
  \begin{equation}\label{eq:expansion}
      \begin{split}
          \Evover{\mathbb{P}(\bomega)}
          {\left(\bI - \trainK^{-1/2}\widetilde{\bK}_{J}\trainK^{-1/2}\right)^{2}}
          &=
          \frac{\mu_{2}}
          {J}
      \end{split}
  \end{equation}
  Therefore, incorporating the previous result into Eq. (\ref{eq:6}), allows us to
  conclude that
  \begin{equation*}
      \begin{split}
          \Evover{\mathbb{P}(\bomega)}{\by^{\top} \widetilde{\bK}_{J}^{-1} \by}
          -
          \by^{\top} \trainK^{-1} \by
          =
          \mathcal{O}\left(1/J\right)
          \\
      \end{split}
  \end{equation*}

  For the model complexity term $\logdet$ we follow a similar procedure as before.
  For a fixed $\bomega$
  \begin{equation*}
      \begin{split}
          &\log\left|\widetilde{\bK}_{J}\right|
          =
          \log\left|\trainK - \trainK + \widetilde{\bK}_{J}\right|
          \\
          &=
          \log\left|\trainK \right|
          -
          \log\left|\bI -
          \left(\bI - \trainK^{-1/2} \widetilde{\bK}_{J} \trainK^{-1/2}\right)
          \right|
      \end{split}
  \end{equation*}
  Focusing on the last term we have that
  \begin{equation*}
      \begin{split}
          &\log\left|\bI -
          \left(\bI - \trainK^{-1/2} \widetilde{\bK}_{J} \trainK^{-1/2}\right)
          \right|
          \\
          &=
          \text{tr} \left(
          \log \left(\bI -
          \left(\bI - \trainK^{-1/2} \widetilde{\bK}_{J} \trainK^{-1/2}\right)
          \right) \right)
      \end{split}
  \end{equation*}
  We can rewrite the term in the R.H.S as follows
  \begin{equation*}
      \begin{split}
          &\log
          \left(\bI -
          \left(\bI - \trainK^{-1/2} \widetilde{\bK}_{J} \trainK^{-1/2}\right)
          \right)
          \\
          &=
          \sum_{t=1}^{\infty}
          \frac{1}{t}
          \left(\bI - \trainK^{-1/2} \widetilde{\bK}_{J}
          \trainK^{-1/2}\right)^{t}
      \end{split}
  \end{equation*}
  Again, following the analysis of \cite{nowozin2018debiasing}, we can express the
  central moments of the random variable
  $\trainK^{-1/2}\widetilde{\bK}_{J}\trainK^{-1/2}$ to the central moments of
  $\trainK^{-1/2}
  \left(\bphi\left(\bomega\right) \bphi\left(\bomega\right)^{\top}
  + \sigma^{2} \bI \right)
  \trainK^{-1/2}$
  denoted as $\mu_{i}$ for $i \geq 2$. Also, as explained before, the dominant term will
  be ($t=2$), therefore we have that thus, using Eq. (\ref{eq:expansion})
  we have that
  \begin{equation*}
      \begin{split}
          &\text{tr}\left(
          \Evover{\mathbb{P}(\bomega)}{
          \log
          \left(\bI -
          \left(\bI - \trainK^{-1/2} \widetilde{\bK}_{J} \trainK^{-1/2}\right)
          \right)} \right)
          \\
          &=
          \mathcal{O}\left(1/J\right)
      \end{split}
  \end{equation*}
  Therefore, we can conclude that
  \begin{equation*}
      \begin{split}
          \Evover{\mathbb{P}(\bomega)}{
          \log\left|\widetilde{\bK}_{J}\right|}
          -
          \log\left|\trainK \right|
          =
          \mathcal{O}\left(1/J\right).
      \end{split}
  \end{equation*}
\end{proof}

\section{Further Derivations of Randomized Truncation Estimators}
\label{RTE-appendix}
Here we prove the unbiasedness of both Russian Roulette (RR) \cref{eqn:rr} and Single Sample (SS) \cref{eqn:ss_estimator} estimators.
Recall from \cref{subsec:randomized_truncation_intro} that we wish to estimate the expensive
$$\textstyle \psi = \sum_{j=1}^H \Delta_j, \quad H \in \mathbb{N} \cup \{\infty\},$$
by computing just the first $J \ll H$ terms, where $J\in \{1,\dots,H\}$ is randomly drawn from the truncation distribution $\mathbb{P}(\mathcal{J}=J)$. 
RR and SS estimators, denoted as $\bar{\psi}$, offer two strategies for up-weighting the $J$ computed terms, which, as we will prove below, yield an unbiased estimator $\mathbb{E}_J [\bar{\psi}] = \psi$.
\subsection{Unbiasedness of the RR Estimator}
The RR estimator of \cref{eqn:rr} is unbiased, i.e., $\mathbb{E}_J  \left[ \bar{\psi}_J \right] = \psi$.
\begin{proof}
\begin{align*}
\mathbb{E}_J \left[ \bar{\psi}_J \right] &= \mathbb{E}_{J} \left[ \sum_{j=1}^J  \frac{\Delta_j}{\mathbb{P} (\mathcal{J} \geq j )}\right] \\
&= \mathbb{E}_{J} \left[\sum_{j=1}^H \frac{\Delta_j}{\mathbb{P} (\mathcal{J} \geq j)} \cdot \mathbb{I}_{j \leq J} \right] \\
&= \sum_{j=1}^H \frac{\Delta_j}{\mathbb{P} (\mathcal{J} \geq j)} \mathbb{E}_{J}[\mathbb{I}_{j \leq J}] \\
&= \sum_{j=1}^H \frac{\Delta_j}{\mathbb{P} (\mathcal{J} \geq j)} \left[\sum_{J=1}^H \mathbb{P} (\mathcal{J} = J) \cdot \mathbb{I}_{J \geq j} \right] \\
&= \sum_{j=1}^H \frac{\Delta_j}{ \mathbb{P} (\mathcal{J} \geq j )}  \mathbb{P} (\mathcal{J} \geq j )
=  \psi
\end{align*}
\end{proof}


\subsection{Unbiasedness of the SS Estimator}
\label{appendix:SS_unbiased}
The SS estimator of \cref{eqn:ss_estimator} is unbiased, i.e., $\mathbb{E}_J \left[ \bar{\psi}_J \right] = \psi$.
\begin{proof}
\begin{align*}
\mathbb{E}_J \left[ \bar{\psi}_J \right] &= \mathbb{E}_J \left[ \sum_{j=1}^H  \frac{\Delta_j}{\mathbb{P} (\mathcal{J} = j )} 
\cdot \mathbb{I}_{j=J} \right] \\
&= \sum_{j=1}^H  \frac{\Delta_j}{\mathbb{P} (\mathcal{J} = j )} \mathbb{E}_J[\mathbb{I}_{j=J}]\\
&= \sum_{j=1}^H  \frac{\Delta_j}{\mathbb{P} (\mathcal{J} = j )} \sum_{J=1}^H \mathbb{P}(\mathcal{J}=J)\cdot \mathbb{I}_{J=j}\\
&= \sum_{j=1}^H  \frac{\Delta_j}{\mathbb{P} (\mathcal{J} = j )} \mathbb{P} (\mathcal{J} = j )
=  \psi
\end{align*}
\end{proof}

\subsection{Minimizing the Variance of the SS Estimator}
Below we will derive the optimal distribution that minimizes the variance of our SS
  estimator. Note that for a given truncation distribution, we have that
  \begin{equation*}
      \begin{split}
        \mathbb{V}_J \left(\bar{\psi}\right)
        &=
        \sum_{j=1}^{H} \frac{\Delta_{j}^{2}}{\mathbb{P}\left(\mathcal{J}=j\right)^{2}}
        \mathbb{V}_J \left(\mathbb{I}_{\mathcal{J}=j}\right)
        \\
        &=
        \sum_{j=1}^{H} \Delta_{j}^{2}
        \left(\frac{1 - \mathbb{P}\left(\mathcal{J}=j\right)}
        {\mathbb{P}\left(\mathcal{J}=j\right)}\right)
      \end{split}
  \end{equation*}
  since $\mathbb{I}_{\mathcal{J}=j}$ is a Bernoulli random variable with probability
  $\mathbb{P}\left(\mathcal{J}=j\right)$, we can plug-in its variance and derive the
  second equality. Hence,
  to find the truncation distribution that minimizes the variance of the SS estimator we
  can solve the following constraint optimization problem.
  \begin{equation*}
      \begin{split}
        \min_{p}
          \sum_{j=1}^{H} \Delta_{j}^{2} \frac{1 - p_{j}}{p_{j}}
          \quad
          \text{s.t.} \quad \sum_{j=1}^{H} p_{j} = 1 , \quad p_{j} \geq 0
      \end{split}
  \end{equation*}
  where $p_{j}$ is acting as a shorthand of $\mathbb{P}\left(\mathcal{J}=j\right)$.  The
  Lagrangian of the problem is
  \begin{equation*}
      \begin{split}
        \mathcal{L}\left(p, \lambda\right)
        =
          \sum_{j=1}^{H}
          \Delta_{j}^{2}
          \frac{1 - p_{j}}{p_{j}}
          + \lambda \left(1 - \sum_{j=1}^{H} p_{j}\right)
      \end{split}
  \end{equation*}
  where we can ignore the nonnegativity constraints in $p_{j}$ as long as $\Delta_{j} >
  0$ (see solution below). Hence, the first order conditions dictate that
  \begin{equation*}
      \begin{split}
        \frac{\partial \mathcal{L}}{\partial p_{j}}\left(p_{j}^{\star}, \lambda^{\star}\right)
        =
        -\left(\frac{\Delta_{j}^{2}}{p_{j}^{\star}}\right)^{2}
        + \lambda^{\star}
        = 0
      \end{split}
  \end{equation*}
  therefore, if we take the ratio of the probabilities with respect to the first we get
  that $p_{j}^{\star} = \frac{\Delta_{j}}{\Delta_{1}} p_{1}^{\star}$. If we substitute
  this expression in the equality constraint we get that
  \begin{equation} \label{eq:optimalss}
      \begin{split}
        \mathbb{P}^{\star}\left(\mathcal{J}=j\right)
        =
        \frac{\Delta_{j}}{\sum_{i=1}^{H} \Delta_{i}}
        \propto
        \Delta_{j}
      \end{split}
  \end{equation}
  We emphasize that this result is a guide for practical choices of the truncation
  distribution. It is impractical to compute it as it will require evaluating all the
  $\Delta_{j}$ for $j=1,\dots,H$. However, if we posses an estimate or a theoretical
  bound on rate of decay of each $\Delta_{j}$ then our unnormalized truncation
  distribution should also decay at this rate to minimize variance.

\subsection{SS estimator as Importance Sampling}
Here we derive the SS estimator by importance sampling the quantity $\psi$ with $\mathbb{P}(\mathcal{J}=J)$ as our proposal distribution.
\begin{align*}
\psi &= \sum_{J=1}^H \Delta_J
\intertext{Next, re-write the summation above as an expectation over the discrete uniform distribution $J \sim \mathcal{U}[1,H] = \frac{1}{H}, \; \forall{J}$:}
 &= H\sum_{J=1}^H \frac{1}{H}\Delta_J, \\
 \intertext{We now introduce an alternative proposal distribution $J \sim \mathbb{P}(\mathcal{J}=J)$:}
&= H\sum_{J=1}^H \frac{\Delta_J}{H \mathbb{P}(\mathcal{J}=J)}  \mathbb{P}(\mathcal{J}=J) \\ 
&= \Evover{J \sim \mathbb{P}(\mathcal{J}=J)}{\frac{\Delta_J}{\mathbb{P}(\mathcal{J}=J)}}, 
\end{align*}
Approximating the final expectation using a single Monte Carlo sample results in the SS estimator.

\section{Estimating the Marginal Log Likelihood from RR-CG}
Recall from \cref{sec:debias_cg} that we use the Russian Roulette estimator in conjunction with conjugate gradients to compute an unbiased (stochastic) gradient of the log marginal likelihood.
In this section, we briefly describe how to obtain an unbiased estimate of the log marginal likelihood itself.

The $\invquad$ term can be estimated directly from the $\trainK^{-1} \by$ solve in \cref{eqn:rrcg-linear-solve}.
The $\logdet$ term is less straightforward, as the estimate from the CG byproducts (Eq.~\ref{eqn:logdet_estimate}) isn't readily expressed as a summation.
Instead, we apply the Russian Roulette estimator to the following telescoping series:
\begin{align}
\logdet &\approx \Vert \bz \Vert^2 \be_1^T \left(\log \bT_\bz^{(N)} \right) \be_1
\label{eqn:cg_logdet_telescope}
\\
&=\Vert \bz \Vert^2 \be_1^T \left(\log \bT_\bz^{(1)} \right) \be_1 + \sum_{j=2}^N \Delta_j,
\nonumber
\end{align}
where
$\Delta_j = \Vert \bz \Vert^2 \be_1^T \left(\log \bT_\bz^{(j)} - \log\bT_\bz^{(j-1)} \right) \be_1$.
We can therefore apply the Russian Roulette estimator with truncation $\mathbb{P}(J)$ to \cref{eqn:cg_logdet_telescope}.
This telescoping series can be expensive to compute.
Since it is unnecessary for gradient-based optimization, we introduce it only as a tool to analyze the RR-CG log marginal likelihood.

\section{Optimal Truncation Distributions}
\subsection{Proof of Theorem~\ref{th:optimalrrcg}} 

\begin{proof}
  For \cref{th:optimalrrcg} we have to combine the results of \cref{CG_bias_theorem} with the optimal distribution
  for RR that is derived in \citet{beatson2019efficient}. We will only focus on
  $v_{J}$ which that approximates $\logdet$ since the procedure is analogous for
  $u_{J}$ which approximates $\invquad$.
  \citet[][Theorem~5.4]{beatson2019efficient} state that the optimal RR truncation
  distribution that maximizes the ROE is
  \begin{equation*}
      \begin{split}
        \mathbb{P}^{*}\left(J \geq j\right)
        \propto
        \sqrt{
          \frac{\mathbb{E}\left[||v_{j}||^{2}\right]}{c\left(j\right) -
          c\left(j-1\right)}
      }
      \end{split}
  \end{equation*}
  where $c\left(j\right)$ is the computational cost of evaluating
  $v_{j}$ which is the $j$-th term being approximated through Russian Roulette.
  In this case, the $v_{j}$ are nonrandom and single-valued and
  the cost per iteration is constant with respect to $j$. Hence, we can conclude that
  \begin{equation*}
      \begin{split}
        \mathbb{P}^{*}\left(J \geq j\right)
        &\propto
        v_{j}
        =
        \mathcal{O}\left(C^{-2j}\right)
      \end{split}
  \end{equation*}
  which implies that $\mathbb{P}^{*}\left(J \leq j\right) \propto
  1 - \mathcal{O}\left(C^{-2j}\right)$
  since the derivative of an exponential function is of the same form we have that
  $\mathbb{P}^{*}\left(J=j\right) \propto \mathcal{O}\left(C^{-2j}\right)$.
\end{proof}

\subsection{Proof of Theorem~\ref{th:optimalssrff}}

The strategy for \cref{th:optimalssrff} is to relate the kernel approximation using $J+1$ random
Fourier features to the kernel approximation using $J$ features for the two components of
the log marginal likelihood: the data-fit term $\invquad$ and the model complexity term
$\logdet$. For $\invquad$ we create this connection through the Sherman-Morisson formula
and for the $\logdet$ we use the matrix determinant lemma. Throughout these proofs, we
will require some results regarding positive definite matrices that we add as remarks
below. Before stating those remarks, we will first introduce some auxiliary notation.
\begin{equation*}
    \begin{split}
        \bK_{J}
        &:=
        \frac{1}{J}\sum_{j=1}^{J}
        \bphi_{j} \bphi_{j}^{\top}
        \\
        \left(J + 1\right)\bK_{J+1}
          &=
          J \bK_{J}
          +
          \bphi_{J+1} \bphi^{\top}_{J+1}
    \end{split}
\end{equation*}
Note the difference between the $\widetilde{\bK}_{J}$ term used throughout the paper
which includes the $\sigma^{2} \bI$ against the $\bK_{J}$ term which only includes
the RFF features. Moreover, to reduce clutter we define the following matrix
\begin{equation*}
    \begin{split}
      \bW := \left(\frac{J}{J+1} \bK_{J} + \sigma^{2} \bI\right)
    \end{split}
\end{equation*}
whose use will become evident throughout the proof.
\begin{rem}
  Given a positive definite matrix $\bA$ we have that for any $a \in \left(0, 1\right)$ and
  any $b > 0$
  \begin{equation*}
      \begin{split}
        \norm{a \bA + b \bI}_{2}^{2}
          \leq
          \norm{\bA + b \bI}_{2}^{2}
      \end{split}
  \end{equation*}
\end{rem}
If we express $\bA = \bV \bD \bV^{\top}$, where $\bV$ is an orthogonal matrix and
$\bD$ is a diagonal matrix containing the positive eigenvalues of $\bA$, then
we have that
$a \bA + b \bI = \bV \left(a \bD + b \bI\right) \bV^{\top}$. This implies that the
eigenvalues of $\bA + b \bI$ are larger than those of $a\bA + b \bI$. An immediate
consequence of this remark is that $\norm{\bW}_{2}^{2} \leq
\norm{\bK_{J} + \sigma^{2} \bI}_{2}^{2}$ or that
$\norm{\bW^{-1}}_2^2 \leq \norm{\Lambda^{-1} + \sigma^{-2}}_2^{2}
\leq \sigma^{-4}$ where $\Lambda$ contains the positive eigenvalues of $\bK_{J}$.
\begin{rem}
  Given a positive definite matrix $\bA$ we have that for any $a \in \left(0,1\right)$,
  any $b > 0$ and any vector
  $\bx$
  \begin{equation*}
      \begin{split}
        \bx^{\top} \left(a\bA + b \bI\right)^{-1} \bx
        \leq
        \bx^{\top} \left(\bA + b \bI\right)^{-1} \bx
        \leq
        b^{-1} \bx^{\top} \bx
      \end{split}
  \end{equation*}
\end{rem}
This remark follows by using the same diagonal decomposition as the one used in the
previous remark. Hence, by noting that the eigenvalues of
$\left(b \bI\right)^{-1}$ are larger than those of $\left(a \bA + b \bI\right)^{-1}$
and this last matrix has larger eigenvalues than $\left(\bA + b \bI\right)^{-1}$ then
the result follows.

\begin{proof}[Proof of Theorem~\ref{th:optimalssrff}]
  We start by showing the rate of decay for the $\Delta_{j}$ involving the data-fit terms.
  Note that we can express
  \begin{equation}\label{eq:deltafit}
      \begin{split}
        &\by^{\top} \left(\bK_{J+1} + \sigma^{2} \bI\right)^{-1} \by
          \\
          &=
          \by^{\top} \left(
          \frac{J}{J+1} \bK_{J} +
          \sigma^{2} \bI +
          \frac{\bphi_{J+1} \bphi^{\top}_{J+1}}{J+1}\right)^{-1} \by
          \\
          &=
          \by^{\top} \left(
          \bW + \frac{\bphi_{J+1} \bphi^{\top}_{J+1}}{J+1}\right)^{-1} \by
      \end{split}
  \end{equation}
  Applying Sherman-Morisson formula to the inverse of the R.H.S results in
  \begin{equation*}
      \begin{split}
        \bW^{-1}
        -
        \frac{\left(\bphi_{J+1}^{\top} \bW^{-1}\right)^{\top}
        \left(\bphi_{J+1}^{\top} \bW^{-1}\right)}
        {\left(J+1\right) + \bphi_{J+1}^{\top} \bW^{-1}\bphi_{J+1}}
      \end{split}
  \end{equation*}
  where the symmetry of $\bW$ allows us to express the numerator as above.
  Substituting the previous result into Eq. (\ref{eq:deltafit}) and rearranging terms
  allows us to conclude that
  \begin{equation}\label{eq:invquad_delta}
      \begin{split}
        &\by^{\top} \left(\sigma^{2} \bI + \bK_{J}\right)^{-1} \by
        -
        \by^{\top} \left(\sigma^{2} \bI + \bK_{J+1}\right)^{-1} \by
        \\
        &\leq
        \by^{\top} \left(\sigma^{2} \bI + \frac{J}{J+1}\bK_{J}\right)^{-1} \by
        -
        \by^{\top} \left(\sigma^{2} \bI + \bK_{J+1}\right)^{-1} \by
        \\
        &=
        \frac{\left(\bphi_{J+1}^{\top} \bW^{-1} \by\right)^{2}}
        {\left(J+1\right) + \bphi_{J+1}^{\top} \bW^{-1}\bphi_{J+1}}
        \\
        &\leq
        \frac{\left(\bphi_{J+1}^{\top} \bW^{-1} \by\right)^{2}}
        {J+1}
        \\
        &\leq
        \frac{
          \norm{\bphi_{J+1}}_{2}^{2} \norm{\bW^{-1} \by}_{2}^{2}
        }
        {J+1}
        \\
        &\leq
        \frac{
          \norm{\bphi_{J+1}}_{2}^{2} \sigma^{-4}\norm{\by}_{2}^{2}
        }
        {J+1}
        \\
      \end{split}
  \end{equation}
  where the first inequality follows from Remark 2,
  the second inequality occurs since $\bphi_{J+1}^{\top} \bW^{-1} \bphi_{J+1} > 0$,
  the third inequality results from applying Cauchy-Schwarz and the fourth stems from Remark 1.
  Finally,
  taking expectations in Eq. (\ref{eq:invquad_delta}) we can conclude that
  \begin{equation*}
      \begin{split}
        &\Evover{\mathbb{P}(\bomega)}{
          \by^{\top} \left(\sigma^{2} \bI + \bK_{J}\right)^{-1} \by}
        \\
        &-
        \Evover{\mathbb{P}(\bomega)}{
          \by^{\top} \left(\sigma^{2} \bI + \bK_{J+1}\right)^{-1} \by}
        \\
        &=
        \mathcal{O}\left(1/J\right).
      \end{split}
  \end{equation*}

  We now move into the rate of decay of the $\Delta_{j}$ terms involving the model
  complexity terms. Note that we can express
  \begin{equation*}
      \begin{split}
        \left| \bK_{J+1} + \sigma^{2} \bI \right|
        &=
        \left|
        \frac{J}{J+1}\bK_{J} + \sigma^{2} \bI +
        \frac{\bphi_{J+1} \bphi_{J+1}^{\top}}{J+1}
        \right|
        \\
      \end{split}
  \end{equation*}
  then by using the matrix determinant lemma we have that
  \begin{equation}\label{eq:madet}
      \begin{split}
        &\left| \bK_{J+1} + \sigma^{2} \bI \right|
        =
        \left( 1 + \frac{1}{J+1} \bphi^{\top}_{J+1} \bW^{-1} \bphi_{J+1} \right)
        \left| \bW \right|
        \\
        &\leq
        \left( 1 + \frac{1}{J+1} \bphi^{\top}_{J+1} \bW^{-1} \bphi_{J+1} \right)
        \left| \bK_{J} + \sigma^{2} \bI \right|
        \\
        &\leq
        \left( 1 + \frac{\sigma^{-2}}{J+1} \bphi^{\top}_{J+1} \bphi_{J+1} \right)
        \left| \bK_{J} + \sigma^{2} \bI \right|
      \end{split}
  \end{equation}
  where the first inequality follows from Remark 1 and from noting that the determinant
  is equivalent to the product of the eigenvalues of the matrix. The second inequality
  follows from Remark 2.
  Now, we will take the logarithms and the expectations of Eq. (\ref{eq:madet}). We will
  first focus on the first term in the RHS. By using a Taylor expansion of the logarithm
  up to a second term we have that
  \begin{equation*} \label{eq:28}
      \begin{split}
        &\Evover{\mathbb{P}(\bomega)}{
        \log \left( 1 + \frac{\sigma^{-2}}{J+1} \bphi^{\top}_{J+1} \bphi_{J+1} \right)
      }
      \\
        &=
        \frac{\sigma^{-2}
        \Evover{\mathbb{P}(\bomega)}{\bphi^{\top}_{J+1} \bphi_{J+1}}}{J+1}
        +
        \mathcal{O}\left(1/J^{2}\right)
      \end{split}
  \end{equation*}
  therefore, substituting the previous result into Eq. (\ref{eq:madet})
  \begin{equation*}
      \begin{split}
        &\Evover{\mathbb{P}(\bomega)}{
          \log \left| \bK_{J+1} + \sigma^{2} \bI \right|}
          -
         \Evover{\mathbb{P}(\bomega)}{
           \log \left| \bK_{J} + \sigma^{2} \bI \right|}
          \\
          &=
          \frac{\sigma^{-2}
          \Evover{\mathbb{P}(\bomega)}{\bphi^{\top}_{J+1} \bphi_{J+1}}}{J+1}
          + \mathcal{O}\left(1/J^{2}\right)
          \\
          &=
          \mathcal{O}\left(1/J\right)
      \end{split}
  \end{equation*}
  which concludes the analysis of the rate of decay of the log determinant terms.  Thus,
  since each rate of decay if $\mathcal{O}\left(1 / J\right)$, then following Eq. (\ref{eq:optimalss})
  we have that the truncation distribution
  that minimizes the variance is
  \begin{equation*}
      \begin{split}
        \mathbb{P}^{*}\left(J\right)
        \propto
        \Delta_{J}
        =
        \mathcal{O}\left(1 / J\right).
      \end{split}
  \end{equation*}
\end{proof}

\begin{table*}[t]
\centering
\resizebox{0.85\linewidth}{!}{\begin{tabular}{ccccccccccccccc}
\toprule
       &  & \multicolumn{7}{c}{RMSE} \\
       &  &                    Cholesky &           POE &              RFF &                        SVGP &                        sgGP &                          CG &                       RR-CG \\
Dataset & $n$ &                             &                             &                             &                             &                             &                             \\
\midrule
PolTele & $9.6$K &             $.112 \pm .002$ & 
$.174 \pm .006 $ & $\mathbf{ .106 \pm .000 }$ &             $.150 \pm .000$ &             $.128 \pm .001$ &             $.119 \pm .002$ &             $.112 \pm .002$ \\
Elevators & $10.6$K &  $\mathbf{ .360 \pm .006 }$ & $.379 \pm .005$ &  $\mathbf{ .365 \pm .001 }$ &             $.376 \pm .006$ &  $\mathbf{ .362 \pm .006 }$ &  $\mathbf{ .360 \pm .006 }$ &  $\mathbf{ .360 \pm .006 }$ \\
Bike & $11.1$K &  $\mathbf{ .035 \pm .003 }$ &   $.071 \pm .002$      &     $.063 \pm .001$ &             $.045 \pm .002$ &            $1.044 \pm .334$ &  $\mathbf{ .040 \pm .005 }$ &  $\mathbf{ .035 \pm .002 }$ \\
\hline
Kin40K & $25.6$K &                         --- &  $.248 \pm .001 $&  $\mathbf{ .074 \pm .000 }$          &   $.152 \pm .001$ &             $.081 \pm .000$ &             $.093 \pm .000$ &             $.091 \pm .001$ \\
Protein & $25.6$K &                         --- &  $.715 \pm .007$ & $\mathbf{ .547 \pm .001 }$ &             $.664 \pm .008$ &             $.562 \pm .005$ &  $\mathbf{ .541 \pm .008 }$ &  $\mathbf{ .541 \pm .008 }$ \\
KEGG & $31.2$K &                         --- &       $.097 \pm .004$  &    $.101 \pm .001$ &  $\mathbf{ .088 \pm .002 }$ &  $\mathbf{ .089 \pm .002 }$ &             $.195 \pm .064$ &  $\mathbf{ .087 \pm .003 }$ \\
KEGGU & $40.7$K &                         --- &    $.125 \pm .001$     &    $.129 \pm .001$ &             $.122 \pm .001$ &             $.123 \pm .001$ &  $\mathbf{ .120 \pm .000 }$ &  $\mathbf{ .120 \pm .000 }$ \\
3DRoad & $278$K &                         --- &  $.675 \pm .001$        &   $.348 \pm .001$ &             $.439 \pm .002$ &             $.285 \pm .003$ &             $.202 \pm .003$ &  $\mathbf{ .114 \pm .013 }$ \\
\bottomrule
\end{tabular}
}
\resizebox{0.85\linewidth}{!}{\begin{tabular}{ccccccccccccccc}
\toprule
       &  & \multicolumn{7}{c}{NLL} \\
       &  &                    Cholesky &  POE &              RFF &                          SVGP &                         sgGP &                          CG &                       RR-CG \\
Dataset & $n$ &                             &                    &                               &                              &                             &                             \\
\midrule
PolTele & $9.6$K &            $-.464 \pm .006$ &  $-.252 \pm .010 $ &   $-.159 \pm .005$ &              $-.442 \pm .004$ &  $\mathbf{ -.480 \pm .004 }$ &            $-.354 \pm .003$ &            $-.458 \pm .006$ \\
Elevators & $10.6$K &  $\mathbf{ .425 \pm .013 }$ & $.469 \pm .016$ &  $.780 \pm .006$ &    $\mathbf{ .442 \pm .015 }$ &   $\mathbf{ .426 \pm .015 }$ &  $\mathbf{ .429 \pm .012 }$ &  $\mathbf{ .425 \pm .013 }$ \\
Bike & $11.1$K &            $-.984 \pm .018$ & $-.799 \pm .010$  &   $.099 \pm .030$ &  $\mathbf{ -1.514 \pm .024 }$ &          $85.715 \pm 49.088$ &            $-.971 \pm .008$ &            $-.982 \pm .018$ \\
\hline
Kin40K & $25.6$K &                         --- &  $.464 \pm .002 $ &  $1.407 \pm .004$ &    $\mathbf{ -.410 \pm .003 }$ &              $.427 \pm .001$ &             $.468 \pm .003$ &             $.449 \pm .013$ \\
Protein & $25.6$K &                         --- & $1.105 \pm .008$ &   $1.163 \pm .003$ &              $1.014 \pm .012$ &              $.951 \pm .004$ &  $\mathbf{ .934 \pm .006 }$ &  $\mathbf{ .934 \pm .006 }$ \\
KEGG & $31.2$K &                         --- & $-.874 \pm .011$  & $6.311 \pm 2.323$ &   $\mathbf{ -1.022 \pm .023 }$ &  $\mathbf{ -.981 \pm .033 }$ &             $.415 \pm .653$ &            $-.884 \pm .009$ \\
KEGGU & $40.7$K &                         --- &   $ -.636 \pm .002 $ & $4.000 \pm .303$ &   $\mathbf{ -.685 \pm .005 }$ &             $-.668 \pm .005$ &            $-.637 \pm .006$ &            $-.650 \pm .004$ \\
3DRoad & $278$K &                         --- &  $1.031 \pm .002$ & $1.317 \pm .004$ &    $\mathbf{ .601 \pm .004 }$ &              $.831 \pm .000$ &  $\mathbf{ .613 \pm .010 }$ &             $.776 \pm .030$ \\
\bottomrule
\end{tabular}
}
  \caption{Root-mean-square-error (RMSE) and negative log-likelihood (NLL) of exact GPs using CG, RRCG and other baselines
  on UCI regression datasets using a constant prior mean and a RBF kernel with independent lengthscale for each dimension.
  All trials were averaged over 3 trials with different splits.
  $N$ and $d$ are the size and dimensionality of the training dataset, respectively.}
	\label{tab:rbf-ard}
\end{table*}

\begin{table*}[t]
\centering
\resizebox{0.85\linewidth}{!}{\begin{tabular}{ccccccccccccccc}
\toprule
       &  & \multicolumn{7}{c}{Training time  (m)} \\
       &  &                    Cholesky &            POE       &          RFF &                        SVGP &                        sgGP &                          CG &                       RR-CG \\
Dataset & $n$ &                             &                 &            &                             &                             &                             &                             \\
\midrule
PolTele & $9.6$K &             $22.417 \pm .035$ &   $1.167 \pm .006$                  &   $.464 \pm .002 $ &             $3.862 \pm .018 $ &     $.629 \pm .006$          &             $12.793 \pm .111$ &             $14.968 \pm .258$ \\
Elevators & $10.6$K &         $30.617 \pm .016 $ &  $1.051 \pm .008 $   &  $.503 \pm .002 $ &             $4.236 \pm .023 $        &     $.696 \pm .001$          &  $14.443 \pm  .080 $ &  $ 16.519 \pm .152$ \\
Bike & $11.1$K &  $34.991 \pm .016 $ &    $1.364 \pm .012$       &   $.476 \pm .014 $ &             $4.261 \pm .023 $        &    $.691 \pm .002$         &  $ 11.634 \pm .131 $ &  $ 13.669 \pm .080 $ \\
\hline
Kin40K & $25.6$K &                         --- & $.930 \pm .006$ &  $.772 \pm .008 $ &             $9.597 \pm .043$        &      $1.559 \pm .010$             &             $11.345 \pm .073$ &             $ 12.823 \pm .114 $ \\
Protein & $25.6$K &                         --- & $.851 \pm .003$ & $.716 \pm .008 $ &             $11.115 \pm .033    $ &    $1.867 \pm .007$       &  $10.507 \pm .044 $ &  $12.142 \pm .052 $ \\
KEGG & $31.2$K &                         --- &   $1.135 \pm .003$    &      $.682 \pm .002$        &  $11.881 \pm .058$          &    $1.786 \pm .009$         &             $22.780 \pm .021$ &  $25.390 \pm .089 $ \\
KEGGU & $40.7$K &                         --- &   $1.140 \pm  .005 $    &      $.835 \pm .001$ &             $15.281 \pm .070$ &    $2.572 \pm .029$        &  $ 34.396 \pm .026$ &  $37.825 \pm .064 $ \\
3DRoad & $278$K &                         --- & $2.176 \pm .022 $       &     $6.089 \pm .031$ &             $104.164 \pm .294$  &   $22.615 \pm .147$     &   $145.396 \pm 1.373$        & $158.657 \pm .554$ \\
\bottomrule
\end{tabular}
}
  \caption{Total training time (in minutes) of exact GPs using CG, RRCG and other baselines
  on UCI regression datasets (see the number of optimization iterations for each method in experiment setup).
  All trials were averaged over 3 trials with different splits.}
	\label{tab:rbf-ard-time}
\end{table*}

\section{Experiment Details.}

\begin{figure*}[t]
\vskip 0.1in
\begin{center}
\includegraphics[scale=0.5]{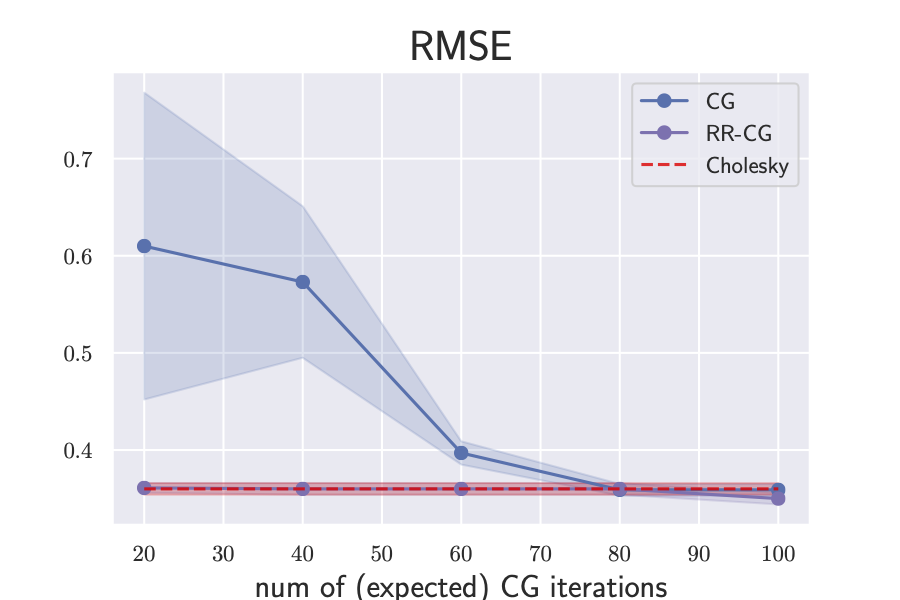}
\includegraphics[scale=0.5]{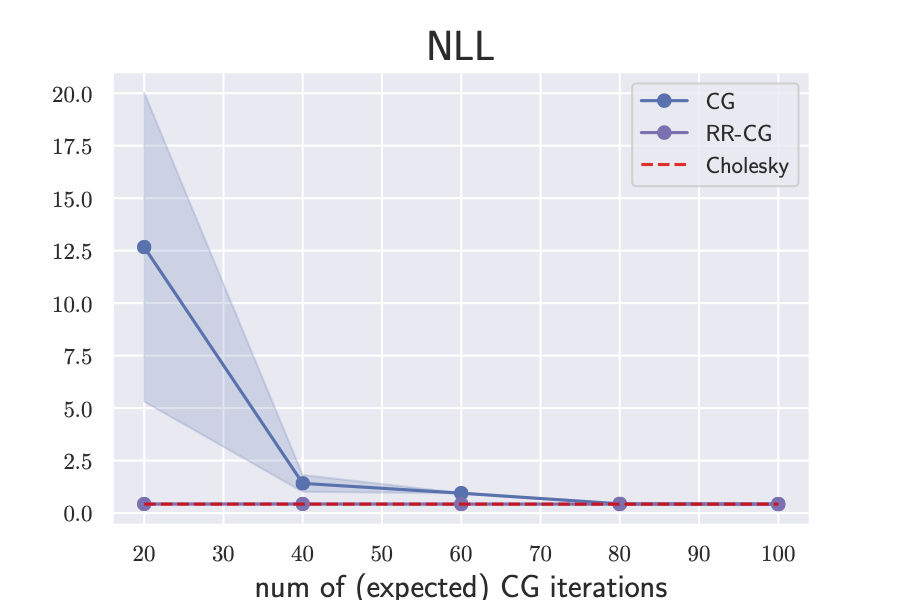}
 \caption{
  Predictive RMSE (left) and NLL (right) as a function of the number of (expected) CG iterations by optimizing Exact GP with CG (blue solid line) and RR-CG (purple solid line). Lower is better.
  The red dashed line corresponds to optimizing GP with Cholesky. The results are over three random seeds. }
\label{fig:debias_cg_rff_predictive}
\end{center}
\vskip -0.2in
\end{figure*}

We provide the experiment details for the predictive performance experiments in \cref{sec:results}. 

\paragraph{Experiment setup.} To optimize the hyperparameters of the CG, RR-CG and Cholesky models, we use an Adam optimizer with learning rate = 0.01 and a MultiStepLR scheduler dropping the learning rate by a factor of 10 at the 50\%, 70\% and 90\% of the optimization iterations. We run the optimization for 1500, 800 and 300 iterations on small (PoleTele, Elevators and Bike), medium (Kin40K, Protein, KEGG and KEGGU) and large (3DRoad) datasets, respectively. The number of iterations for CG and the expected number of iterations for RR-CG are both set to 100. The latter is achieved by using the truncation distribution from \cref{eqn:exp-decay-dist} with $\lambda = 0.05$ and $J_{\min}=80$. 

For CG and RR-CG, we use the rank-5 pivoted Cholesky preconditioner of \citet{gardner2018gpytorch}.
To reduce the number of optimization steps needed for the 3DRoad dataset, we initialize the hyperparameters to those found with the (biased) sgGP method.
During evaluation, we compute the predictive mean using $1{,}000$ iterations of CG.
Predictive variances are estimated using the rank-100 Lanczos approximation of \citet{pleiss2018constant}.

For RFF we use 1{,}000 random Fourier features across all experiments. In terms of optimization, we use an Adam optimizer with learning rate = 0.005 for KEGG, 0.001 for KEGGU and 0.01 for the remaining datasets. We also use a MultiStep scheduler that activates at 85\%, 90\%, 95\% of the optimization iterations with a decay rate of 0.5 and also take a total of 500 optimization iterations on all the datasets.

In all the SVGP models we use $1{,}024$ inducing points, with a full-rank multivariate Gaussian variational distribution.
Hyperparameters and variational parameters are jointly optimized for 300 epochs using minibatches of size $1{,}024$.
As with the other baselines, we use the an Adam optimizer with a learning rate of $0.01$, dropping the learning rate by a factor of $10$ after 50\%, 70\% and 90\% of the optimization iterations.

For sgGP, we train using minibatches of $16$ data points.
As suggested by \citet{chen2020stochastic}, the minibatches are constructed by sampling one training data point and selecting its $15$ nearest neighbors.
To accelerate optimization, we accumulate the gradients of $1{,}024$ minibatches before performing an optimization step (these $1{,}024$ minibatch updates can be performed in parallel, enabling GPU acceleration).
We optimize the models for $300$ epochs, using the same learning rate and scheduler as with SVGP.
During evaluation, we use the same procedure as for CC/RR-CG ($1{,}000$ iterations of CG, rank-$100$ Lanczos variance estimates).

For POE, we divide the training dataset into disjoint subsets, each with $M=1024$ data points.
If $N$ is not divisible by $1024$, we pad the final subset with elements from the first subset.
We train independent GP with independent hyperparameters on each of the subsets using the same optimization procedure as Cholesky models.
Following \citet[][Eqs.~11 and 12]{deisenroth2015distributed}, the posterior distribution for a test input $\bx^*$ is Gaussian with mean $\mu_\text{POE}^*(\bx^*)$ and variance $\sigma^{2*}_\text{POE}(\bx^*)$:
\begin{align*}
    \mu^*_\text{POE}(\bx^*) &= \frac{\sigma^{2*}_\text{POE}(\bx^*)}{K} \sum_{k=1}^K \sigma^{-2*}_k(\bx^*) \mu^*_k(\bx^*),
    \\
    \sigma^{-2*}_\text{POE}(\bx^*) &= \frac{1}{K} \sum_{k=1}^K \sigma^{-2*}_k(\bx^*),
\end{align*}
where $K$ is the number of independent GP experts and $\mu^*_k(\cdot)$ and $\sigma^{2*}_k(\cdot)$ are the posterior mean and variance of each expert model.

\paragraph{Full tables.}In Table \ref{tab:rbf-ard},  we report the RMSE and NLL numbers which are used to plot \cref{fig:nll_rmse} in the paper, and in Table \ref{tab:rbf-ard-time} we report the corresponding training time. 

\section{Additional Experiments}

\subsection{Predictive Performance with Different CG iterations}
 
 Here we include an additional experiment to show that early truncated CG can be detrimental to GP learning while RR-CG remains robust to the expected truncation number. 

We conduct GP learning with the CG, RR-CG and Cholesky methods in the Elevators dataset. We vary the number of (expected) iterations for CG and RR-CG from 20 to 100 and plot the corresponding RMSE / NLL in \cref{fig:debias_cg_rff_predictive}. From the figure, we conclude that: 1) early truncation of the CG algorithm impedes GP optimization and leads to poor predictions. This problem lessens as we increase the number of CG iterations from 20 to 100. 2) The RR-CG model is robust to the expected truncation number, as it keeps comparable RMSE and NLL values to the Cholesky model under different expected truncation numbers. This experiment can only be run in smaller datasets to be able to compare against Cholesky. We chose Elevators as this dataset requires the largest number of CG iterations before convergence hence
making the difference of choosing between CG and RR-CG evident. We emphasize that RR-CG is a better alternative to early-truncated CG as the latter can perform poorly in some cases, whereas the former is robust to the choice of expected truncation number.

\subsection{Convergence of GP hypeparameters for other datasets.}

In this section, we show how SS-RFF and RR-CG converge to the Cholesky solution whereas the biased methods do not. We make this analysis for other datasets than the ones used in the main manuscript. 

Using RFF with $1{,}000$ or $1{,}500$ features generates results that clearly diverge from the Cholesky solution and RFF with a $1{,}000$ features generates numerically unstable training.
In contrast, SS-RFF with $1{,}500$ features is able to achieve the Cholesky solution after $1{,}000$ iterations. Nonetheless, SS-RFF with $1{,}000$ features also suffers from numerical instability as RFF $1{,}000$ but at a lesser degree.
 \begin{figure}[h]
\vskip 0.1in
\begin{center}
\includegraphics[scale=0.4]{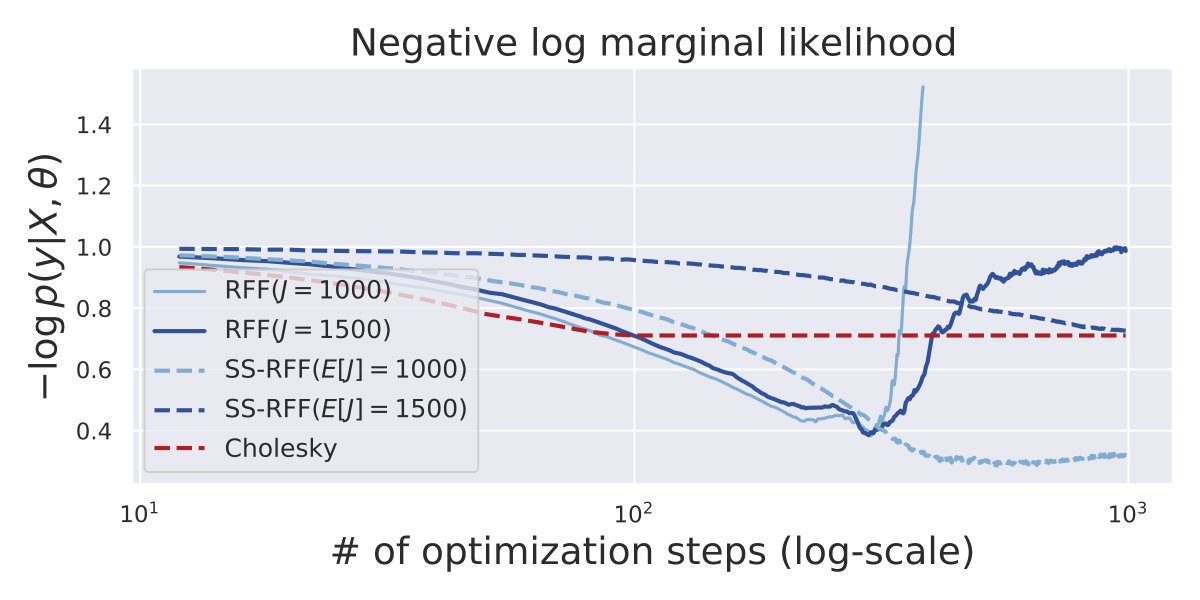}
\caption{
The GP optimization objective for models trained with RFF and SS-RFF. 
(Bike dataset, RBF kernel, Adam optimizer.)
RFF models converge to sub-optimal log marginal likelihoods.
SS-RFF models converge to (near) optimum values, yet require more than $100\times$ as many optimization steps. 
}
\label{fig:loss_evolution_rff_bike}
\end{center}
\vskip -0.2in
\end{figure}

The RR-CG models converge to optimal solutions, while the (biased) CG models diverge. Only the results for a low number of iterations is plotted since for more then 40 iterations CG and RR-CG are indistinguishable from Cholesky.

\begin{figure}[h!]
\vskip 0.1in
\begin{center}
\includegraphics[scale=0.4]{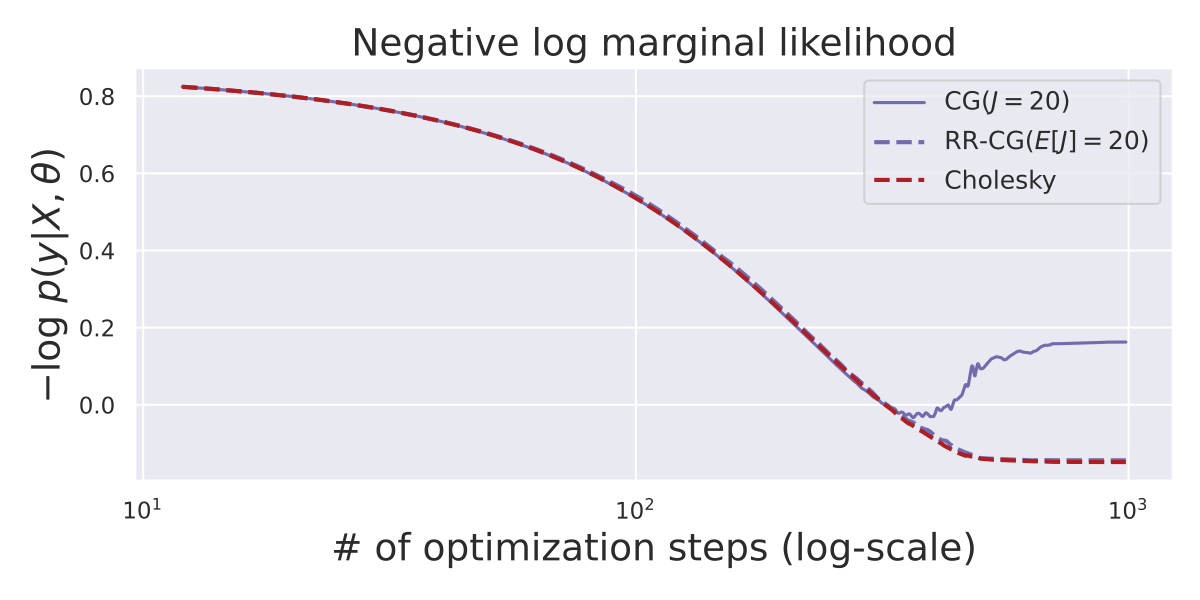}
\caption{
The GP optimization objective for models trained with CG and RR-CG. 
(PoleTele dataset, RBF kernel, Adam optimizer.)
Models converge in $<100$ steps of Adam.
}
\label{fig:loss_evolution_cg_pol}
\end{center}
\vskip -0.2in
\end{figure}
\end{document}